\documentclass{acm_proc_article-sp}
\newfont{\mycrnotice}{ptmr8t at 7pt}
\newfont{\myconfname}{ptmri8t at 7pt}
\let\crnotice\mycrnotice%
\let\confname\myconfname%

\usepackage{graphicx} 
\usepackage{epsfig}
\usepackage{algorithm}
\usepackage{algorithmic}
\newtheorem{theorem}{Theorem}[section]
\newtheorem{lemma}{Lemma}[section]

\DeclareMathOperator*{\argmin}{arg\,min}
\usepackage{amsmath}
\usepackage{amssymb}

\begin{document}
\title{Robust Subspace Clustering via Tighter Rank Approximation}
\numberofauthors{3} 
%
\author{
%
%
\alignauthor
Zhao Kang\\
       \affaddr{Computer Science Dept. }\\
    \affaddr{Southern Illinois University}\\
     \affaddr{Carbondale, IL, USA}\\
    \email{zhao.kang@siu.edu}
\alignauthor
Chong Peng\\
    \affaddr{Computer Science Dept. }\\
    \affaddr{Southern Illinois University}\\
     \affaddr{Carbondale, IL, USA}\\
       \email{pchong@siu.edu}
\alignauthor 
Qiang Cheng\\
    \affaddr{Computer Science Dept. }\\
    \affaddr{Southern Illinois University}\\
     \affaddr{Carbondale, IL, USA}\\
      \email{qcheng@cs.siu.edu}
}

\maketitle
\begin{abstract}
Matrix rank minimization problem is in general NP-hard. The nuclear norm is used to substitute the rank function in many recent studies. Nevertheless, the nuclear norm approximation adds all singular values together and the approximation error may depend heavily on the magnitudes of singular values. This might restrict its capability in dealing with many practical problems. In this paper, an arctangent function is used as a tighter approximation to the rank function. We use it on the challenging subspace clustering problem. For this nonconvex minimization problem, we develop an effective optimization procedure based on a type of augmented Lagrange multipliers (ALM) method. Extensive experiments on face clustering and motion segmentation show that the proposed method is effective for rank approximation. 
\end{abstract}




\section{Introduction}
\label{intro}
Matrix rank minimization arises in control, machine learning, signal processing and other areas \cite{zhao2012approximation}. It is difficult to solve due to the discontinuity
and nonconvexity of the rank function. Existing algorithms are largely based on the nuclear norm heuristic, i.e., to replace the rank by the nuclear norm \cite{fazel2002matrix}. The nuclear norm of a matrix $X$, denoted by $\left\|X\right\|_*$, is the sum of all its singular values, i.e., $\left\|X\right\|_*=\sum_i \sigma_i(X)$. Under some conditions, the solution to the nuclear norm heuristic coincides with the minimum rank solution \cite{recht2010guaranteed,recht2011null}. However, since the nuclear norm is the convex envelop of rank($X$) over the unit ball $\lbrace X: \left\|X\right\|_2\leq1\rbrace$, it may deviate from the rank of $X$ in many circumstances \cite{candes2010power,candes2009exact}.  The rank function counts the number of nonvanishing singular values, while the nuclear norm sums their amplitudes. As a result, the nuclear norm may be dominated by a few very large singular values. Variations of standard nuclear norm are shown to be promising in some recent research \cite{hu2013fast,cai2010singular,nie2012low}. A number of nonconvex surrogate functions have come up to better approximate the rank function, such as Logarithm Determinant \cite{fazel2002matrix,kang2015logdet}, Schatten-$p$ norm \cite{mohan2012iterative}, truncated nuclear norm \cite{hu2013fast} and others \cite{lu2014generalized}.
In general, they are to solve the following low-rank minimization problem: 
\begin{equation}
\label{generaldefinite}
\min_Z \sum_{i=1}^{\min (m, n)}h(\sigma_i(Z))+\lambda g(Z),
\end{equation}
where $\sigma_i(Z)$ denotes the $i$-th singular value of $Z\in \mathbf{\mathcal{R}}^{m\times n}$, $h(\cdot)$ is a potentially nonconvex, nonsmooth function, and $g(\cdot)$ is a loss function. By choosing $h(z)=z$, the summation of the first term in (\ref{generaldefinite}) goes back to the nuclear norm $\left\|Z\right\|_*$, problem (\ref{generaldefinite}) becomes the well known convex relaxation of the rank minimization problem:
\begin{equation}
\label{rank}
\min_Z\hspace{.2cm} \|Z\|_*+\lambda g(Z).
\end{equation}
In this paper, we will propose a new nonconvex rank approximation and consider subspace clustering as a specific application.  

\subsection{Previous Work on Subspace Clustering}
In many real-world applications, high-dimensional data reside in a union of multiple low-dimensional subspaces rather than one single low-dimensional subspace \cite{elhamifar2009sparse}. Subspace clustering deals with exactly this structure by clustering data points according to their underlying subspaces. It has numerous applications in computer vision \cite{rao2010motion} and image processing \cite{ma2007segmentation}. Therefore subspace clustering has drawn significant attention in recent years \cite{vidal2010tutorial}. In practice, the underlying subspace structure is often corrupted by noise and outliers, and thus the data may deviate from the original subspaces. It is necessary to develop robust estimation techniques. 

A number of approaches to subspace clustering have been proposed in the past two decades. According to the survey in \cite{vidal2010tutorial}, they can be roughly divided into four categories: 1) algebraic methods; 2) iterative methods; 3) statistical methods; and 4) spectral clustering-based methods. Among them, spectral clustering-based methods have obtained state-of-the-art results, including sparse subspace clustering (SSC) \cite{elhamifar2013sparse}, and low rank representation (LRR) \cite{liu2013robust}. They perform subspace clustering in two steps: first, learning an affinity matrix that encodes the subspace membership information, and then applying spectral clustering algorithms \cite{shi2000normalized, ng2002spectral} to the learned affinity matrix to obtain the final clustering results. Their main difference is how to obtain a good affinity matrix.

SSC assumes that each data point can be represented as a sparse linear combination of other points. The popular $l_1$-norm heuristic is used to capture the sparsity. It enjoys great performance for face clustering and motion segmentation data. Now we have a good theoretical understanding about SSC. For instance, \cite{elhamifar2010clustering} shows that disjoint subspaces can be exactly recovered under certain conditions; geometric analysis of SSC \cite{soltanolkotabi2012geometric} significantly broadens the scope of SSC to intersecting subspaces. However, the data points are assumed to be lying exactly in the subspace. This assumption may be violated in the presence of corrupted data. \cite{wang2013noisy} extends SSC by adding adversarial or random noise. However, SSC's solution might be too sparse, thus the affinity graph from a single subspace will not be a fully connected body \cite{nasihatkon2011graph}. To address the above issue, another regularization term is introduced to promote connectivity of the graph \cite{elhamifar2013sparse}.

LRR also represents each data point as a linear combination of other points. It is to find the lowest rank representation $Z$ of all data points jointly, where the nuclear norm is used as a common surrogate of the rank function. In the presence of noise or outliers, LRR solves the following problem:
\begin{equation}
\min_{Z, E} \left\|Z\right\|_*+\lambda \left\|E\right\|_l \hspace{.5cm} s.t. \hspace{.5cm} X=XZ+E,
\end{equation}
where $\lambda$ balances the effects of the low rank representation and errors, $X=[x_1, \cdots, x_n]\in \mathbf{\mathcal{R}}^{m\times n}$ is a set of $m$-dimensional data vectors drawn from the union of $k$ subspaces $\lbrace S_i \rbrace _{i=1}^k$, and $\left\| \cdot\right\|_l$ characterizes certain corruptions $E$. For example, when $E$ represents Gaussian noise, squared Frobenius norm $\|E\|_F^2=\sum\limits_{i=1}^m \sum\limits_{j=1}^n E_{ij}^2$ is used; when $E$ denotes random corruptions, $l_1$ norm $\|E\|_1=\sum\limits_{i=1}^m \sum\limits_{j=1}^n |E_{ij}|$ is appropriate; when $E$ indicates the sample-specific corruptions, $l_{2,1}$ norm is adopted, where $\|E\|_{2,1}=\sum\limits_{j=1}^n \sqrt{\sum\limits_{i=1}^m E_{ij}^2}$. 
The low rank as well as sparsity requirement may help counteract corruptions.
A variant of LRR works even in the presence of some arbitrarily large outliers \cite{liu2011exact}. However, LRR has never been shown to succeed other than under strong ``independent subspace'' condition \cite{kanatani2001motion}.

In view of the issues with the nuclear norm mentioned in the beginning, we propose the use of an arctangent function instead in this work. We demonstrate the enhanced performance of the proposed algorithm on benchmark data sets. 
\subsection{Our Contributions}
 In summary, the main contributions of this paper are threefold:
\begin{itemize}
\item{More accurate rank approximation is proposed to obtain the low-rank representation of high-dimensional data.}
\item{An efficient optimization procedure is developed for arctangent rank minimization (ARM) problem. Theoretical analysis shows that our algorithm converges to a stationary point.}
\item{The superiority of the proposed method to various state-of-the-art subspace clustering algorithms is verified with significantly and consistently lower error rates of ARM on popular datasets.}
\end{itemize}

\begin{figure}[h]
\vskip 0.2in
\begin{center}
\includegraphics[width=.48\columnwidth]{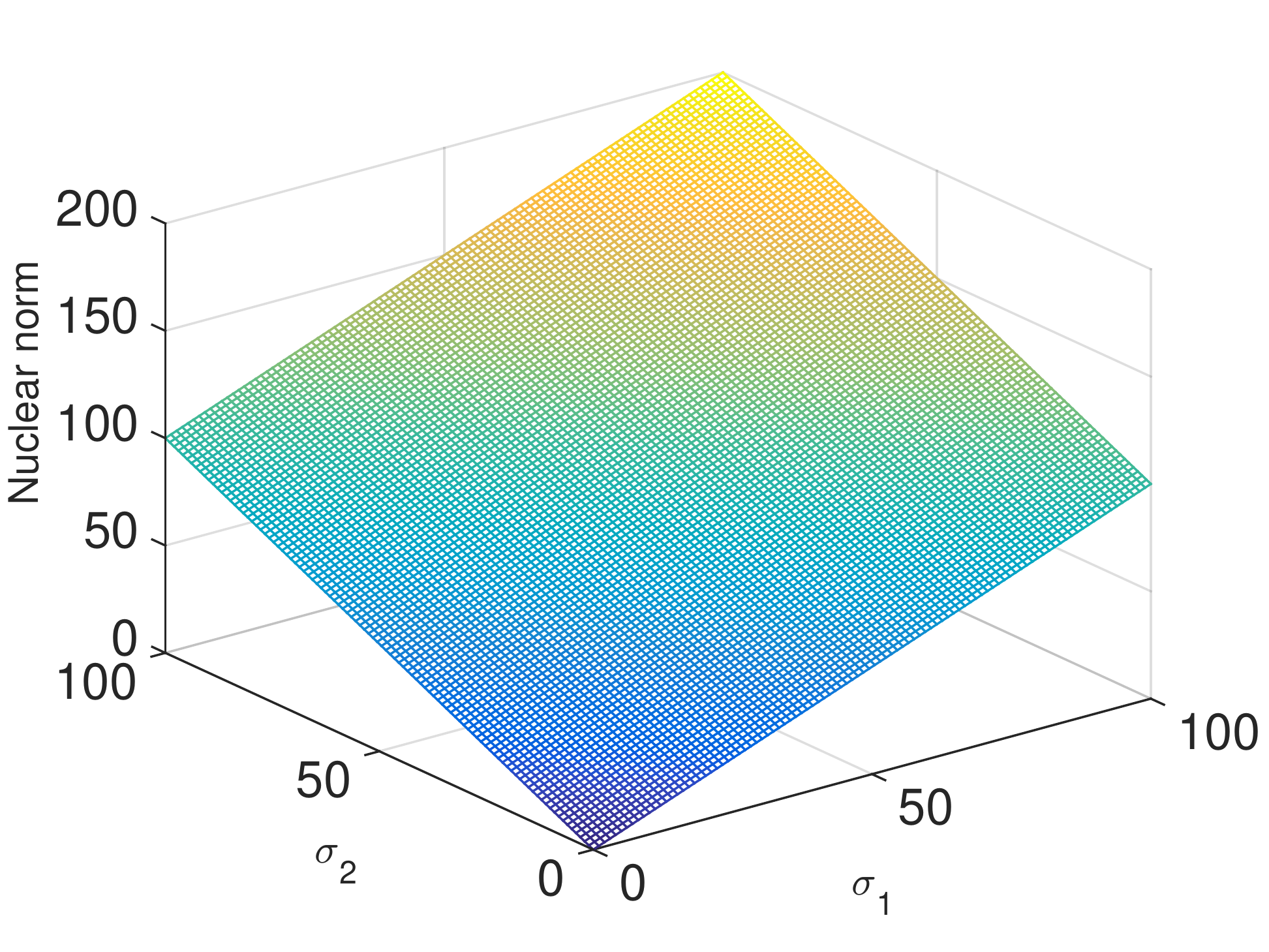}
\includegraphics[width=.48\columnwidth]{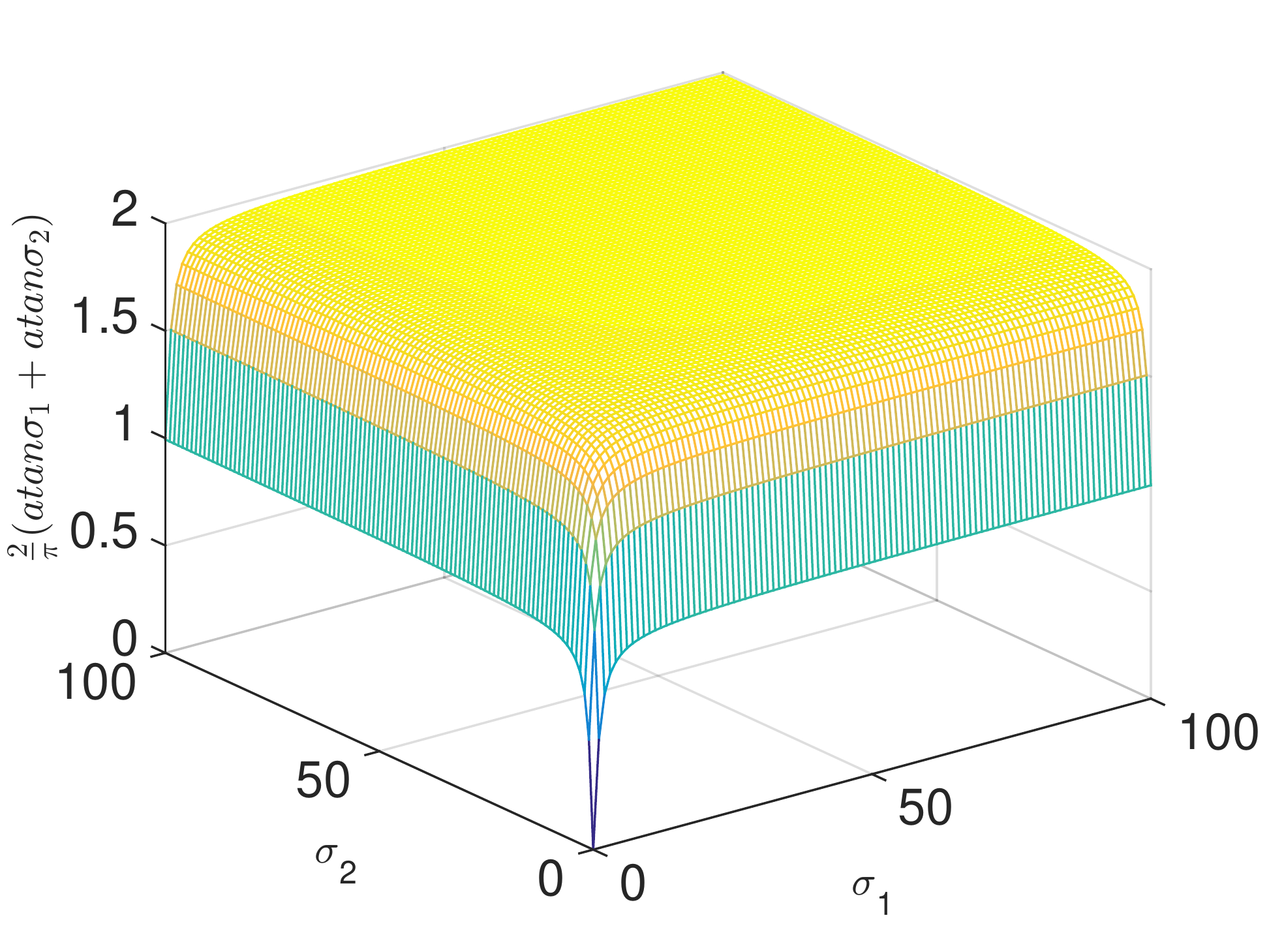}
\caption{Comparison of approximation for rank 2.}
\label{threed}
\end{center}
\vskip -0.2in
\end{figure} 
\section{Subspace clustering by ARM}
In this work, we demonstrate the application of $F(Z)=f\circ\sigma(Z)$, where $f(x)=\sum_i\arctan(|x_i|)$ for any $x \in {\mathcal{R}}^n$ (for high dimensional data usually $m >> n$ which is the case we suppose in the paper), as a rank approximation of matrix $Z$ in subspace clustering setting.
There are three advantages of this approximation function. First, it approximates rank($Z$) much better than the nuclear norm does, i.e., as $\sigma_i \in [0, \infty]$, $\frac{2}{\pi}\arctan (\sigma_i)\in[0,1] $. Figure~\ref{threed} shows the rank approximation value of the two approaches for rank 2 situation. We can clearly see that arctangent reflects the real rank pretty well on a broad range of singular values. Second, $f$ is differentiable, concave and monotonically increasing on $[0, \infty]^n$, by defining the gradient of $f$ at 0 as $\frac{\partial}{\partial x_i} f(0) =\lim_{x_i \rightarrow 0^+} \frac{1}{1+x_i^2} = 1$. Third, $F$ is unitarily invariant and $f$ is absolutely symmetric, i.e., $f(x)$ is invariant under arbitrary permutation and sign changes of the components of $x$. Based on these properties, we have the following theorems, which is proved in Appendix A.

\begin{theorem}
\label{firsthmm}
For $\mu>0$ and $A\in \mathbf{\mathcal{R}}^{m\times n} $
, the following problem
\begin{equation}
Z^*=\argmin_Z F(Z)+\frac{\mu}{2}\left\|Z-A\right\|_F^2
\label{theoremprob}
\end{equation}
is solved by the vector minimization 
\begin{equation}
\label{vectorf}
\sigma^*=\argmin_{\sigma\geq0} f(\sigma)+\frac{\mu}{2}\|\sigma-\sigma_A\|_2^2,
\end{equation}
so that $Z^* = U diag(\sigma^*) V^T $ with the SVD of $A$ being\\
 $U diag(\sigma_A) V^T$. 

\end{theorem}
\subsection{Arctangent Rank Minimization}
To demonstrate the effectiveness of the arctangent rank approximation, we consider its application in the challenging subspace clustering problem. We propose the following arctangent rank minimization (ARM) problem:
\begin{equation}
\label{original}
\min_{Z, E} \hspace{.1cm} \sum_{i=1}^n arctan(\sigma_i(Z))+\lambda\left\|E\right\|_{l}\hspace{.2cm} s.t.\hspace{.2cm} X=XZ+E.
\end{equation} 
It is difficult to solve (\ref{original}) directly because the objective function is neither convex nor concave. We convert it to the following equivalent problem:
\begin{equation}
\label{prob}
\min_{Z, E, J} \hspace{.01cm} \sum_{i=1}^n arctan(\sigma_i(J))+\lambda \left\|E\right\|_l\hspace{.02cm}s.t.\hspace{.02cm} X=XZ+E,\hspace{.02cm}Z=J.
\end{equation}
Now we resort to a type of augmented Lagrange multipliers (ALM) \cite{lin2011linearized} method to solve (\ref{prob}). For simplicity of notation, we denote $\sigma_i=\sigma_i(J)$ and $\sigma_i^t=\sigma_i(J^t)$. The corresponding augmented Lagrangian function is:
\begin{equation}
\begin{split}
&L(E, J, Y_1, Y_2, Z, \mu)=\sum_{i=1}^n arctan(\sigma_i)+\lambda \left\|E\right\|_l\\
&+Tr(Y_1^T(X-XZ-E))+Tr(Y_2^T(J-Z))\\
&+\frac{\mu}{2}(\left\|X-XZ-E\right\|_F^2+\left\|J-Z\right\|_F^2),
\end{split}
\end{equation}
where $\mu>0$ is a penalty parameter and $Y_1$, $Y_2$ are Lagrangian multipliers. The variables $E$, $J$, and $Z$ can be updated alternatively, one at each step, while keeping the other two fixed. For the ($t+1$)th iteration, the iterative scheme is given as follows.

For $Z^{t+1}$,  by fixing $E^t$, $J^t$, $Y_1^t$ and $Y_2^t$, we have:
\begin{equation}
\begin{split}
Z^{t+1}&= \argmin_Z \hspace{.1cm} Tr[(Y_1^t)^T(X-XZ-E^{t})]+\\
&Tr[(Y_2^t)^T(J^{t}-Z)]+\\
&\frac{\mu^t}{2}(\left\|X-XZ-E^{t}\right\|_F^2+\left\|J^{t}-Z\right\|_F^2).
\end{split}
\label{upZ}
\end{equation}
It is evident that the objective function of (\ref{upZ}) is a strongly convex quadratic function which can be solved directly. By setting the first derivative of it to zero, we have:
\begin{equation}
\label{solveZ}
Z^{t+1}=(I+X^TX)^{-1}[X^T(X-E^{t})+J^{t}+\frac{X^TY_1^t+Y_2^t}{\mu^t}],
\end{equation}
 where $I\in \mathbf{\mathcal{R}}^{n\times n}$ is the identity matrix.

For $J^{t+1}$, we have: 
\begin{equation}
\label{arct}
J^{t+1}= \argmin_J \hspace{.2cm}\sum_{i=1}^n arctan(\sigma_i)+
\frac{\mu^t}{2}\|J-(Z^{t+1}-\frac{1}{\mu^t}Y_2^t)\|_F^2.
\end{equation}
Then we can convert it to problem (\ref{vectorf}). The first term in (\ref{vectorf}) is concave while the second term convex in $\sigma$, so we can apply difference of convex (DC) \cite{horst1999dc} (vector) optimization method. A linear approximation is used at each iteration of DC programing. At iteration $k+1$,  
\begin{equation}
\label{vectorized}
\sigma_{k+1}=\argmin_{\sigma\geq 0} \quad \langle w_k,\sigma \rangle+ \frac{\mu^t}{2}\|\sigma-\sigma_A\|_2^2,
\end{equation}
whose closed-form solution is 
\begin{equation}
\label{optisigma} 
\sigma_{k+1}=(\sigma_A-\frac{\omega_k}{\mu^t})_+,
\end{equation}
where $\omega_k=\partial f(\sigma_k)$ is the gradient of $f(\cdot)$ at $\sigma_k$ and the SVD of $Z^{t+1}-\frac{Y_2^t}{\mu^t}$ is $U diag\lbrace\sigma_A\rbrace V^T$. Finally, it converges to a local optimal point $\sigma^*$. Then $J^{t+1}=U diag\lbrace\sigma^*\rbrace V^T$.

For $E^{t+1}$, we have the following subproblem:
\begin{equation}
\begin{split}
E^{t+1}&= \argmin_E \hspace{.1cm}\lambda \left\|E\right\|_l+\frac{\mu^t}{2}\|X-XZ^{t+1}-E\|_F^2\\
&+Tr[(Y_1^{t})^T(X-XZ^{t+1}-E)].
\end{split}
\end{equation}
Depending on different regularization strategies, we have different closed-form solutions. For squared Forbenius norm, it is again a quadratic problem, 
\begin{equation}
\label{error2}
E^{t+1}=\frac{Y_1^{t}+\mu^t(X-XZ^{t+1})}{\mu^t+2\lambda}.
\end{equation}
For $l_1$ and $l_{2,1}$ norm, we use the lemmas from Appendix B. 
Let $Q=X-XZ^{t+1}+\frac{Y_1^t}{\mu^t}$, we can solve $E$ element-wisely as below:  
\begin{eqnarray}
\label{error1}
E_{ij}^{t+1} = \left\{
\begin{array}{ll} Q_{ij}-\frac{\lambda}{\mu^t}sgn(Q_{ij}) , 
&\mbox{if $ |Q_{ij}| <\frac{\lambda}{\mu^t}$};\\
0, &\mbox{otherwise.}
\end{array}\right. 
\end{eqnarray}
In the case of $l_{2,1}$ norm, we have
\begin{eqnarray}
\label{error21}
[E^{t+1}]_{:,i}=\left\{
\begin{array}{ll} \frac{\left\|Q_{:,i}\right\|_2-\frac{\lambda}{\mu^t}}{\left\|Q_{:,i}\right\|_2}Q_{:,i}, & \mbox{if $\left\|Q_{:,i}\right\|_2>\frac{\lambda}{\mu^t}$};\\
0, & \mbox{otherwise.}
\end{array}\right.
\end{eqnarray}
The update of Lagrange multipliers is: 
\begin{align}
\label{multi1}
Y_1^{t+1}&=Y_1^t+\mu^t(X-XZ^{t+1}-E^{t+1}),\\
\label{multi2}
Y_2^{t+1}&=Y_2^t+\mu^t(J^{t+1}-Z^{t+1}).
\end{align}
The procedure is outlined in Algorithm 1. 
\begin{algorithm}[tb]
   \caption{Arctan Rank Minimization}
   \label{alg:rankminimization}
    {\bfseries Input:} data matrix $X\in \mathbf{\mathcal{R}}^{m\times n}$, parameters $\lambda>0, \mu^0>0$, and $\rho>1$.\\
 {\bfseries Initialize:} $J=I\in \mathbf{\mathcal{R}}^{n\times n}$, $E=0$, $ Y_1=Y_2=0$.\\
  {\bfseries REPEAT}
\begin{algorithmic}[1]
  \STATE Update $Z$ by (\ref{solveZ}). 
  \STATE Solve (\ref{arct}). \\
   \STATE Solve $E$ by either (\ref{error2}), (\ref{error1}) or (\ref{error21}) according to  $l$.
\STATE Update $Y_1$ by (\ref{multi1}) and $Y_2$ by (\ref{multi2}). 
\STATE Update $\mu$ by $\mu^{t+1}=\rho\mu^t$.
\end{algorithmic}
   \textbf{UNTIL} {stopping criterion is met.}
\end{algorithm}
 
\begin{figure}[h]
\vskip 0.2in
\begin{center}
\centerline{\includegraphics[width=\columnwidth]{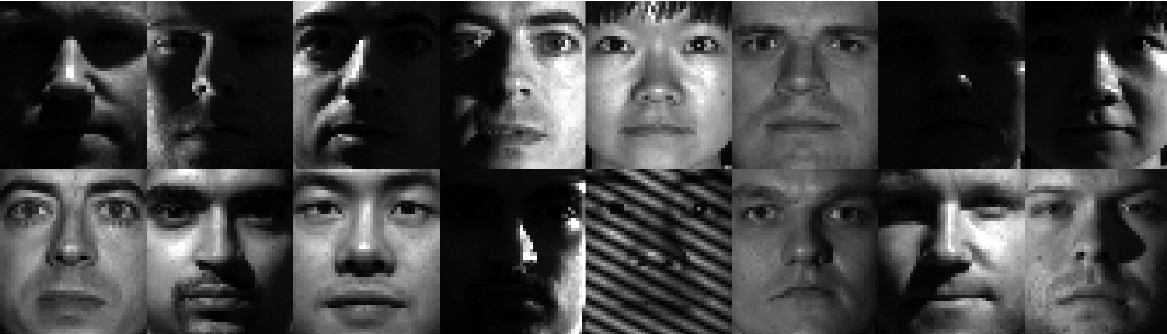}}
\caption{Sample face images in Extended Yale B.}
\label{examples}
\end{center}
\vskip -0.2in
\end{figure} 

\subsection{Affinity Graph Construction} 
After obtaining optimal $Z^*$, we can build the similarity graph matrix $W$. As argued in \cite{elhamifar2013sparse}, some postprocessing of the coefficient matrix can improve the clustering performance. Following the angular information based technique of \cite{liu2013robust}, we define $\widetilde{U}=U\Sigma^{1/2}$, where $U$ and $\Sigma$ are from the skinny $SVD$ of $Z^*=U\Sigma V^T$. 
Inspired by \cite{lauer2009spectral}, we define $W$ as follows:
\begin{equation}
\label{graphmatrix}
W_{ij}=(\frac{\widetilde{u}_i^T\widetilde{u}_j}{\left\|\widetilde{u}_i\right\|_2\left\|\widetilde{u}_j\right\|_2})^{2\alpha},
\end{equation}
where $\widetilde{u}_i$ and $\widetilde{u}_j$ denote the $i$-th and $j$-th columns of $\widetilde{U}$, and $\alpha \in {\mathcal{N^*}}$ controls the sharpness of the affinity between two points. Increasing the power $\alpha$ enhances the separation ability in the presence of noise. However, an excessively large $\alpha$ would break affinities between points of the same group. In order to compare with LRR\footnote{As we confirmed with an author of \cite{liu2013robust}, the power 2 of its equation (12) is
a typo, which should be 4.}, we use $\alpha=2$ in our experiments, then we have the same postprocessing procedure as LRR.
After obtaining $W$, we directly utilize a spectral clustering algorithm NCuts \cite{shi2000normalized} to cluster the samples. Algorithm 2 summarizes the complete subspace clustering steps of the proposed method.
\begin{algorithm}[tb]
\label{alg:completealgo}
   \caption{Subspace Clustering by ARM}
 {\bfseries Input:} data matrix $X$, number of subspaces $k$.\\
  {\bfseries Do}
\begin{algorithmic}[1]
\STATE Obtain optimal $Z^*$ by solving (\ref{prob}).
\STATE Compute the skinny SVD $Z^*=U \Sigma V^T$.
\STATE Calculate $\widetilde{U}=U(\Sigma)^{1/2}$.
\STATE Construct the affinity graph matrix $W$ by (\ref{graphmatrix}).
\STATE Perform NCuts on $W$.
\end{algorithmic}
\end{algorithm} 
\begin{figure}[ht]
\vskip 0.2in
\begin{center}
\includegraphics[width=.45\columnwidth]{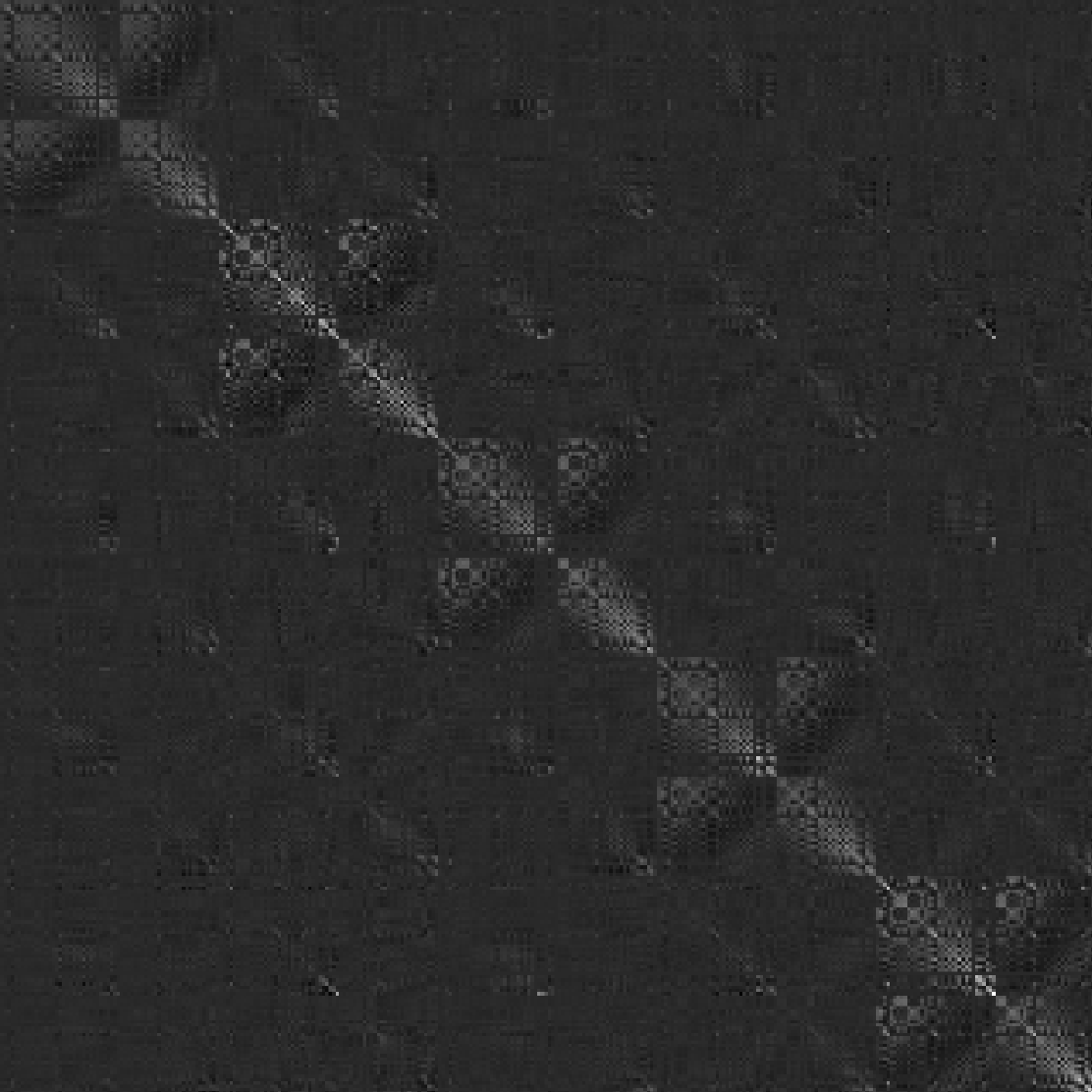}
\includegraphics[width=.45\columnwidth]{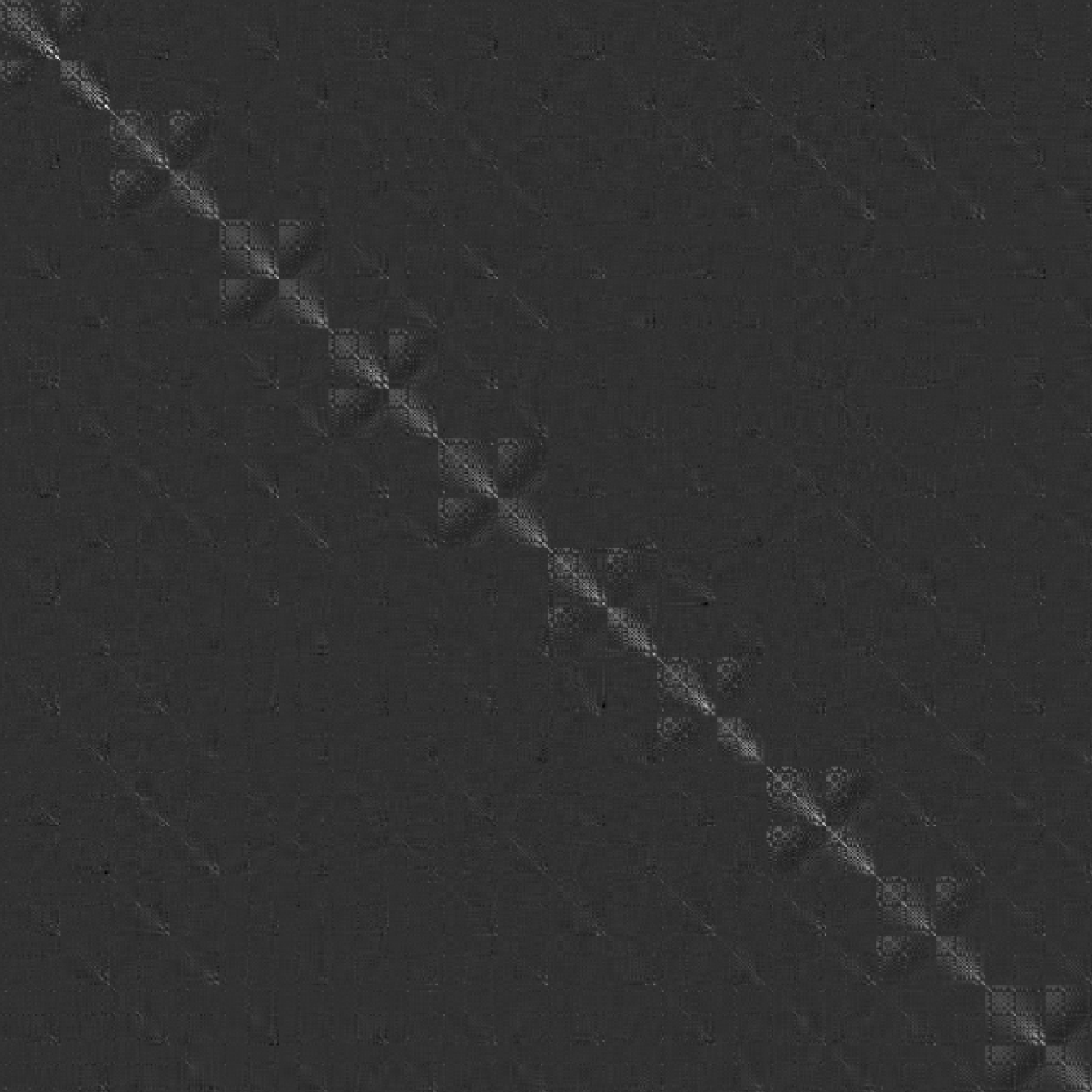}
\caption{ Affinity graph matrix $W$ with five and ten subjects.}
\label{10subject}
\end{center}
\vskip -0.2in
\end{figure} 
\section{Convergence analysis}
Since the objective function (\ref{original}) is nonconvex, it would not be easy to prove the convergence in theory. In this paper, we mathematically prove that our optimization algorithm has at least a convergent subsequence which converges to an accumulation point, and moreover, any accumulation point of our algorithm is a stationary point. Although the final solution might be a local optimum, our results are superior to the global optimal solution from convex approaches. Some previous work also reports similar observations \cite{xiang2013efficient,gong2012multi,zhang2012nonconvex}.  

We will show the proof in the case of $\|E\|_1$. Let's first reformulate our objective function:
\begin{equation}
\begin{split}
G(J, Z, E)&=F(J)+\lambda\|E\|_{1}\quad \\
&s.t.\quad Z=J,\quad X=XZ+E,   \label{second}
\end{split}
\end{equation}
\begin{equation}
\begin{split}
&L(J, Z, E, Y_1, Y_2, \mu)=G(J, Z, E)+\left\langle Y_1,X-XZ-E \right\rangle\\
&+\left\langle Y_2, J-Z\right\rangle 
+\frac{\mu}{2}(\|J-Z\|_{F}^2+\|X-XZ-E\|_{F}^2). \label{eq:AugLag}
\end{split}
\end{equation}
\begin{lemma}
\label{lemma11}
The sequences of $\{Y_{1}^t\}$ and $\{Y_{2}^t\}$ are bounded.
\end{lemma}
\begin{proof}
$J^{t+1}$ satisfies the first-order necessary local optimality condition,
\begin{equation}
\begin{split}
\label{bound1}
&\partial_J L\left(J, Z^{t+1}, E^t, Y_1^t, Y_2^t, \mu^{t}\right)|_{J^{t+1}}\\
=&\partial_J F\left(J\right)|_{J^{t+1}}+\mu^{t}\left(J^{t+1}-Z^{t+1}+\frac{Y_2^t}{\mu^t}\right)\\
=&\partial_J F\left(J\right)|_{J^{t+1}}+Y_2^{t+1}\\
=&0.
\end{split}
\end{equation}
Let's define $\theta_i=\frac{1}{1+(\sigma_i)^2}$ if $\sigma_i \neq 0$; otherwise, it is 1.
According to (\ref{takederi}) in Appendix B, 
\begin{equation}
\begin{split}
Udiag(\theta)V^T=\partial_J F\left(J\right)|_{J^{t+1}}
,
\end{split}
\end{equation}
and $0<\frac{1}{1+(\sigma_i^{t+1})^2}\leq 1$, $\partial_J F\left(J\right)|_{J^{t+1}}$ is bounded. From (\ref{bound1}), we conclude that $Y_2^{t+1}$ is bounded.

Similarly, for $E^{t+1}$ 
\begin{equation}
\begin{split}
&0 \in \partial_E L\left(J^{t+1}, Z^{t+1}, E, Y_1^t, Y_2^t, \mu^{t}\right)|_{E^{t+1}}\\
&=\partial_E \lambda(\|E\|_1)|_{E^{t+1}}-Y_1^t\\
&-\mu^{t}\left(X-XZ^{t+1}-E^{t+1}\right)\\
&=\partial_E \lambda(\|E\|_1)|_{E^{t+1}}+Y_1^{t+1}.
\end{split}
\label{derive}
\end{equation}
Here $\partial_E$ denotes the subgradient operator \cite{clarke1990optimization}. Because $||E||_1$ is nonsmooth only at $E_{ij}=0$, we define $\left[\partial_E\|E\|_1\right]_{ij}=0$ if $E_{ij}=0$. Then $0\leq \| \partial_E \|E\|_1 \|_F^2\leq mn $ is bounded. Therefore, $\{Y_1^{t+1}\}$ is bounded. 
\end{proof}
\begin{lemma}
\label{lemma12}
$\{J^t\}$, $\{E^t\}$ and $\{Z^t\}$ are bounded if $\sum \frac{\mu^{t+1}}{(\mu^t)^2}<\infty$, $\sum \frac{1}{\mu^t}<\infty$ and $X^T X$ is invertible.
\end{lemma}
\begin{proof}
\begin{equation}
\begin{split}
& L\left(J^t, Z^{t},E^{t},Y_1^{t}, Y_2^t, \mu^{t}\right)\\
&=L\left(J^t, Z^{t},E^{t},Y_1^{t-1}, Y_2^{t-1}, \mu^{t-1}\right)+\\
&\frac{\mu^t-\mu^{t-1}}{2}(\|J^t-Z^t\|_F^2+\|X-XZ^t-E^t\|_F^2)\\
&+Tr((Y_1^t-Y_1^{t-1})(X-XZ^t-E^t))\\
&+Tr((Y_2^t-Y_2^{t-1})(J^t-Z^t))\\
&= L\left(J^t, Z^{t},E^{t},Y_1^{t-1}, Y_2^{t-1}, \mu^{t-1}\right)+\\
&\frac{\mu^t+\mu^{t-1}}{2(\mu^{t-1})^2}(\|Y_1^t-Y_1^{t-1}\|_F^2+\|Y_2^t-Y_2^{t-1}\|_F^2).
\end{split}
\end{equation}
\begin{equation}
\begin{split}
& L\left(J^{t+1}, Z^{t+1},E^{t+1},Y_1^{t}, Y_2^t, \mu^{t}\right)\\
&\leq L(J^{t+1}, Z^{t+1}, E^{t}, Y_1^{t}, Y_2^{t}, \mu^t)\\
&\leq L(J^{t}, Z^{t+1}, E^{t}, Y_1^{t}, Y_2^{t}, \mu^t)\\
&\leq L\left(J^{t}, Z^{t},E^{t},Y_1^{t}, Y_2^t, \mu^{t}\right)\\
&\leq  L\left(J^{t}, Z^{t},E^{t},Y_1^{t-1}, Y_2^{t-1}\mu^{t-1}\right)+\\
&\frac{\mu^t+\mu^{t-1}}{2(\mu^{t-1})^2}(\|Y_1^t-Y_1^{t-1}\|_F^2+\|Y_2^t-Y_2^{t-1}\|_F^2). \label{lasteq}
\end{split}
\end{equation}
Iterating the inequality (\ref{lasteq}) gives that
\begin{equation}
\begin{split}
&L(J^{t+1}, Z^{t+1}, E^{t+1}, Y_1^{t}, Y_2^{t}, \mu^t)  \\
&\leq L\left(J^1, Z^{1},E^{1},Y_1^{0}, Y_2^{0}, \mu^{0}\right)+\\
&\sum_{i=1}^t\frac{\mu^i+\mu^{i-1}}{2(\mu^{i-1})^2}(\|Y_1^i-Y_1^{i-1}\|_F^2\\
&+\|Y_2^i-Y_2^{i-1}\|_F^2).
\end{split}
\end{equation}
Under the given conditions on $\{ \mu^t \}$, both terms on the right-hand side of the above inequality  are bounded, thus $L\left(J^{t+1}, Z^{t+1},E^{t+1},Y_1^{t}, Y_2^t, \mu^{t}\right)$ is bounded.
In addition, 
\begin{equation}
\begin{split}
\label{boundseq}
& L\left(J^{t+1}, Z^{t+1},E^{t+1},Y_1^{t}, Y_2^t, \mu^{t}\right)+\\
&\frac{1}{2\mu^t}(\|Y_1^t\|_F^2+\|Y_2^t\|_F^2)\\
&=F(J^{t+1})+\lambda \|E^{t+1}\|_1+\\
&\frac{\mu^t}{2}\|J^{t+1}-Z^{t+1}+\frac{Y_2^t}{\mu^t}\|_F^2+\\
&\frac{\mu^t}{2}\|X-E^{t+1}-XZ^{t+1}+\frac{Y_1^t}{\mu^t}\|_F^2.
\end{split}
\end{equation}
The left-hand side in the above equation is bounded and each term on the right-hand side is nonnegative, so each term is bounded. Therefore, $E^{t+1}$ is bounded. $XZ^{t+1}$ is bounded according to the last term on the right-hand side of (\ref{boundseq}), and thus after multiplying a constant matrix $X^T$, we have $X^T X Z^{t+1}$ is bounded. Under the condition that $X^T X$ is invertible, by multiplying a constant matrix $(X^T X)^{-1}$, we have that $(X^T X)^{-1} X^T X Z^{t+1} = Z^{t+1}$ is bounded. Finally, $J^{t+1}$ is bounded because the second to the last term is bounded. Therefore, $\{J^{t}\}$, $\{E^t\}$ and $\{Z^{t}\}$ are bounded. 
\end{proof}
\begin{theorem}
The sequence $\{J^t, E^t, Z^t, Y_1^t, Y_2^t\}$ generated by Algorithm 1 has at least one accumulation point, under the conditions that 
$\sum \frac{\mu^{t+1}}{ ( \mu^t)^2 } < \infty$, $\sum \frac{1} {  \mu^t } < \infty$ and $X^TX$ is invertible. For any accumulation point  $\{J^*, E^*, Z^*, Y_1^*, Y_2^*\}$, $\{ J^*, E^*, Z^* \}$ is a stationary point of optimization problem (\ref{second}), under the conditions that 
$\mu^t ( J^{t+1} - J^{t}) \rightarrow 0$, and  $\mu^t ( E^{t+1} - E^{t}) \rightarrow 0$. 
\end{theorem}
\begin{proof}
Based on the conditions on the penalty parameter sequence $\{\mu^t\}$ and $X^TX$, Algorithm 1 generates a  bounded sequence $\{J^t, E^t, Z^t, Y_1^t, Y_2^t\}$ by Lemma (\ref{lemma11}) and (\ref{lemma12}) . By the Bolzano-Weierstrass theorem, at least one accumulation point exists, e.g., $\{J^*, E^*, Z^*, Y_1^*, Y_2^*\}$. Without loss of generality, we assume that 
$\{  J^t, E^t, Z^t, Y_1^t, Y_2^t\}$ itself converges to 
$\{ J^*, E^*, Z^*, Y_1^*, Y_2^* \}$. As shown below, $\{ J^*, E^*, Z^*\}$ is a stationary point of problem (\ref{second}), under additional conditions that 
$\mu_t (E^{t+1} - E^{t}) \rightarrow 0$ and $\mu_t (J^{t+1} - J^{t}) \rightarrow 0$.

Since $J^t-Z^t=\frac{Y_2^t-Y_2^{t-1}}{\mu^{t-1}}$ and $\mu^t\rightarrow \infty$, we have $J^t-Z^t \rightarrow 0$. Therefore, $J^*=Z^*$. 

Similarly, by $X-XZ^t-E^t=\frac{Y_1^t-Y_1^{t-1}}{\mu^{t-1}}$,
we have $E^*=X-XZ^*$. 

For $Z^{t+1}$, the first-order optimality condition is 
\begin{equation*}
\begin{split}
&\nabla_Z L\left(J^{t}, Z, E^{t}, Y_1^t, Y_2^t, \mu^{t}\right)|_{Z^{t+1}}\\
&=-Y_2^t-X^TY_1^t-\mu^{t}\left(J^{t}-Z^{t+1}\right)\\
&-\mu^{t}X^T\left(X-E^{t}-XZ^{t+1}\right)\\
&=-Y_2^t-X^TY_1^t-\mu^{t}\left(J^{t+1}-Z^{t+1}+J^t-J^{t+1}\right)\\
&-\mu^{t}X^T\left(X-E^{t+1}-XZ^{t+1}-E^{t}+E^{t+1}\right)\\
&=-Y_2^{t+1}+\mu^{t}\left(J^{t+1}-J^t\right)-X^TY_1^{t+1}\\
&-\mu^{t}X^T\left(E^{t+1}-E^t\right)\\
&=0.
\end{split}
\label{derive}
\end{equation*}
If $\mu^{t}(J^{t+1}-J^{t})\rightarrow 0$ and $\mu^{t}(E^{t+1}-E^{t})\rightarrow 0$, we have
$X^TY_2^t+Y_1^t\rightarrow 0$, i.e., $-X^TY_2^*=Y_1^*$. It is easy to verify $\partial_E L\left(J, Z, E, Y_1, Y_2, \mu\right)|_{E^*}=0$. 
Therefore, $\{J^*, E^*, Z^*, Y_1^*, Y_2^*\}$ satisfies the KKT conditions of $L(J, E, Z, Y_1, Y_2 )$ and thus $\{  J^*, E^*, Z^*  \}$  is a stationary point of (\ref{second}).
\end{proof}
\section{Experiments} 
This section presents experiments with the proposed algorithm on the Extended Yale B (EYaleB) \cite{lee2005acquiring}  and Hopkins 155 databases \cite{tron2007benchmark}. They are standard tests for robust subspace clustering algorithms. As shown in \cite{elhamifar2013sparse}, the challenge in the Hopkins 155 dataset is due to the small principal angles between subspaces. For EYaleB, the challenge lies in the small principal angles and another factor that data points from different subspaces are close. Our results are compared with several state-of-the-art subspace clustering algorithms, including LRR \cite{liu2013robust}, SSC \cite{elhamifar2013sparse}, LRSC \cite{favaro2011closed,vidal2014low}, spectral curvature clustering (SCC) \cite{chen2009spectral}, and local subspace affinity (LSA) \cite{yan2006general}, in terms of misclassification rate\footnote{The implementation of our algorithm is available at: https://github.com/sckangz/arctangent.}. For fair comparison, we follow the experimental setup in \cite{elhamifar2013sparse} and obtain the results. 

As other methods do, we tune our parameters to obtain the best results. In general, the value of $\lambda$ depends on prior knowledge of the noise level of the data. If the noise is heavy, a small $\lambda$ should be adopted. $\mu^0$ and $\rho$ affect the convergence speed. The larger their values are, the fewer iterations are required for the algorithm to 
converge, but meanwhile we may lose some precision of the final objective function value. In the literature, the value of $\rho$ is often chosen between 1 and 1.1. The iteration stops at a relative normed difference of $10^{-5}$ between two successive iterations, or a maximum of 150 iterations.   

\subsection{Face Clustering}

Face clustering refers to partitioning a set of face images from multiple individuals to multiple subspaces according to the identity of each individual. The face images are heavily contaminated by sparse gross errors due to varying lighting conditions, as shown in Figure~\ref{examples}. Therefore, $\left\|E\right\|_1$ is used to model the errors in our experiment. The EYaleB database contains cropped face images of 38 individuals taken under 64 different illumination conditions. 
The 38 subjects were divided into four groups as follows: subjects 1 to 10, 11
to 20, 21 to 30, and 31 to 38. All choices of $n \in \{2, 3, 5, 8, 10\}$ are considered for
each of the first three groups, and all choices of $n \in \left\{2, 3, 5, 8\right\}$ are considered for the last group. As a result, there are $\{163, 416, 812, 136, 3\}$ combinations corresponding to different $n$. Each image is downsampled to $48\times 42$ and is vectorized to a 2016-dimensional vector. $\lambda=10^{-5}$, $\mu^0=1.7$ and $\gamma=1.03$ are used in this experiment. 
\begin{figure}[ht]
\vskip 0.2in
\begin{center}
\begin{tabular}{r c l }
\includegraphics[width=.28\columnwidth]{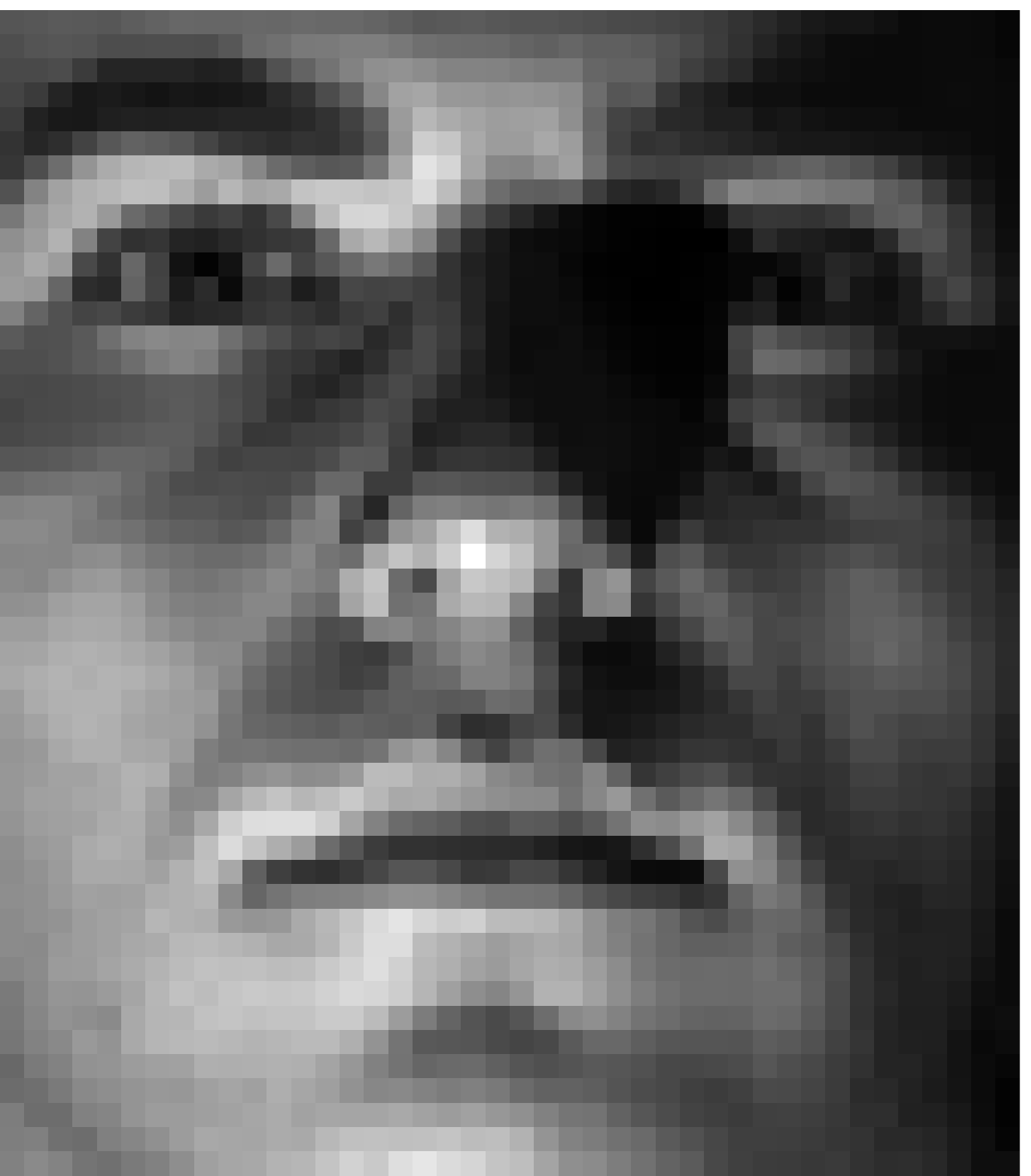}&
\includegraphics[width=.28\columnwidth]{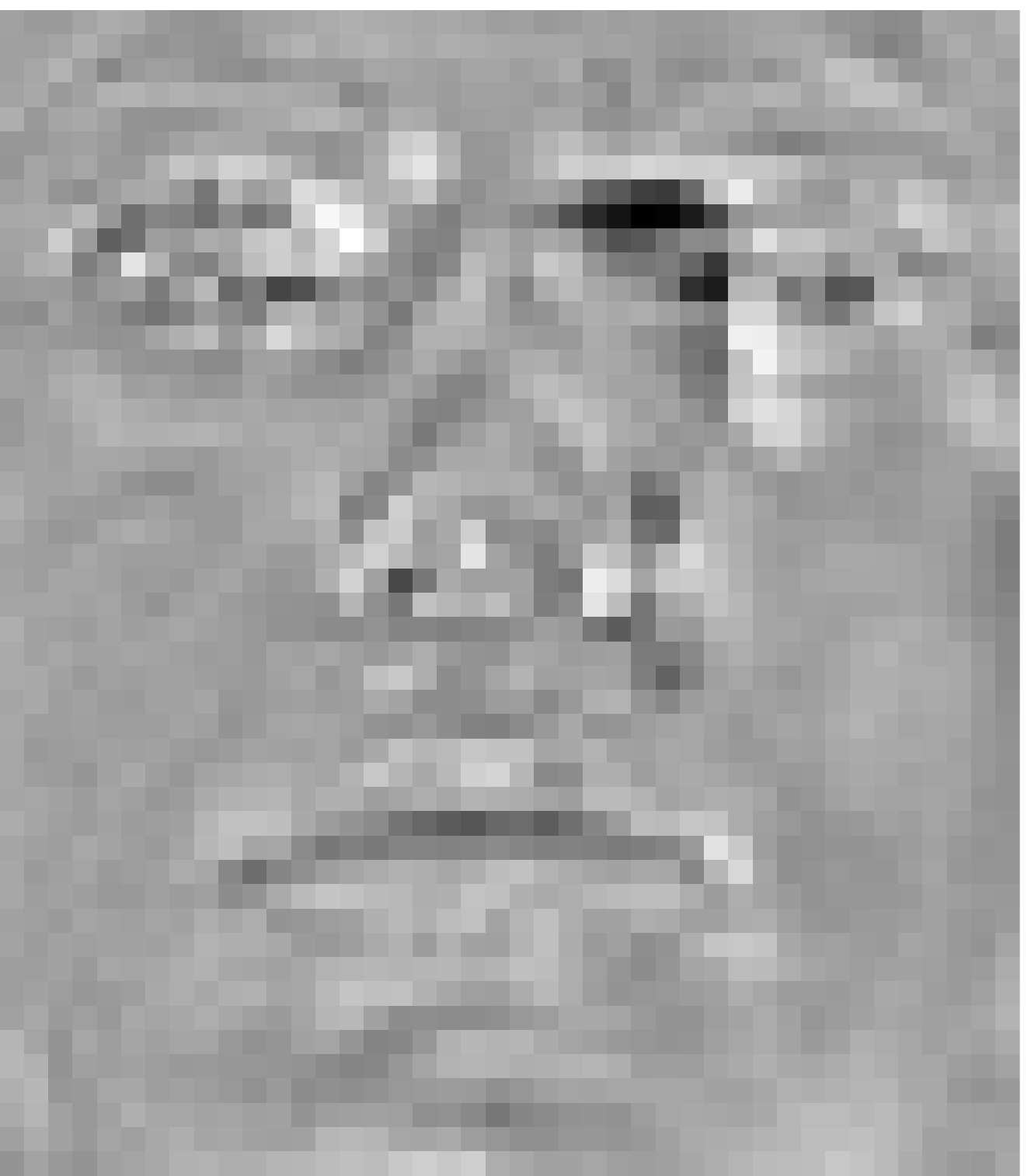}&
\includegraphics[width=.28\columnwidth]{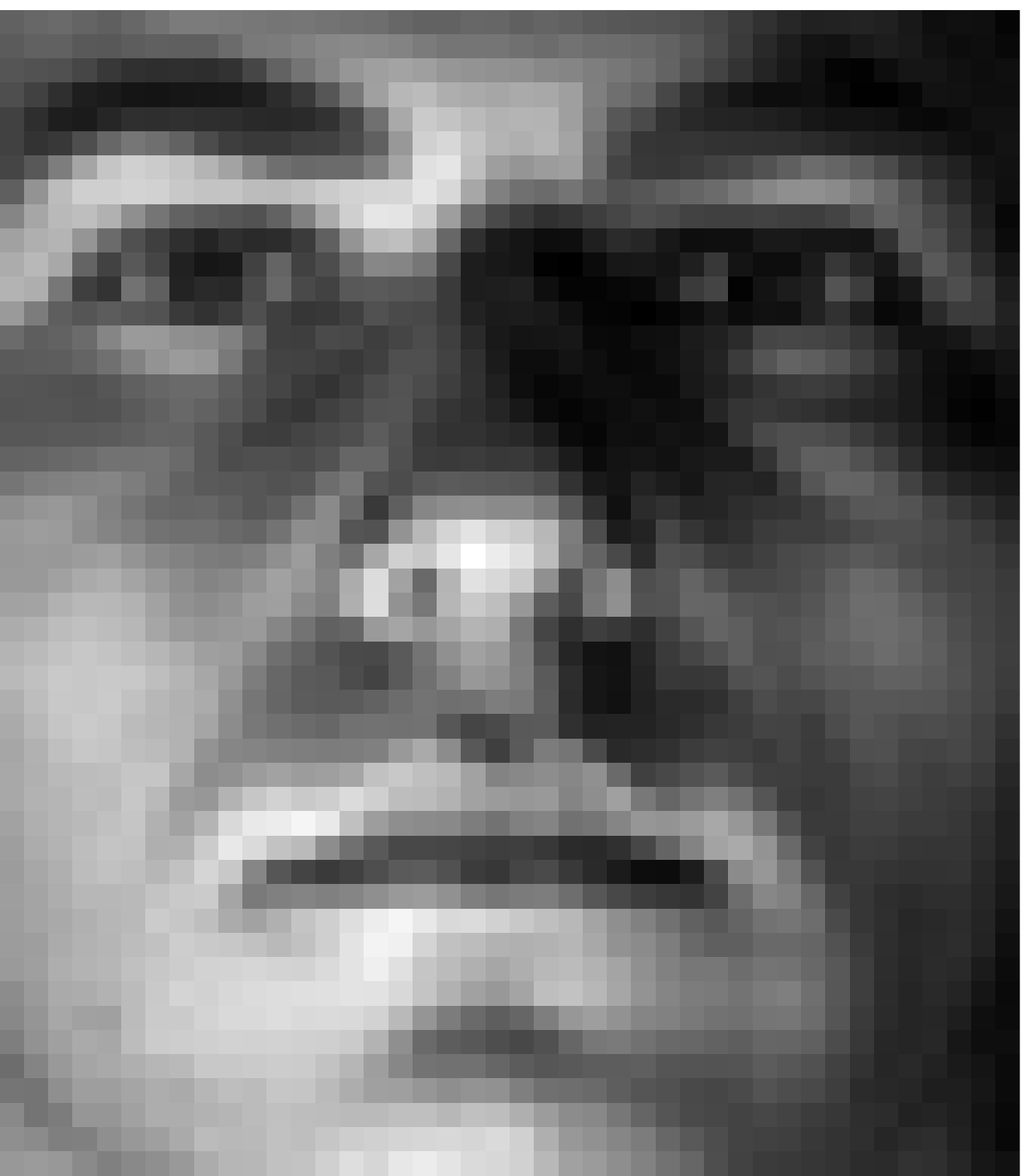}\\
\includegraphics[width=.28\columnwidth]{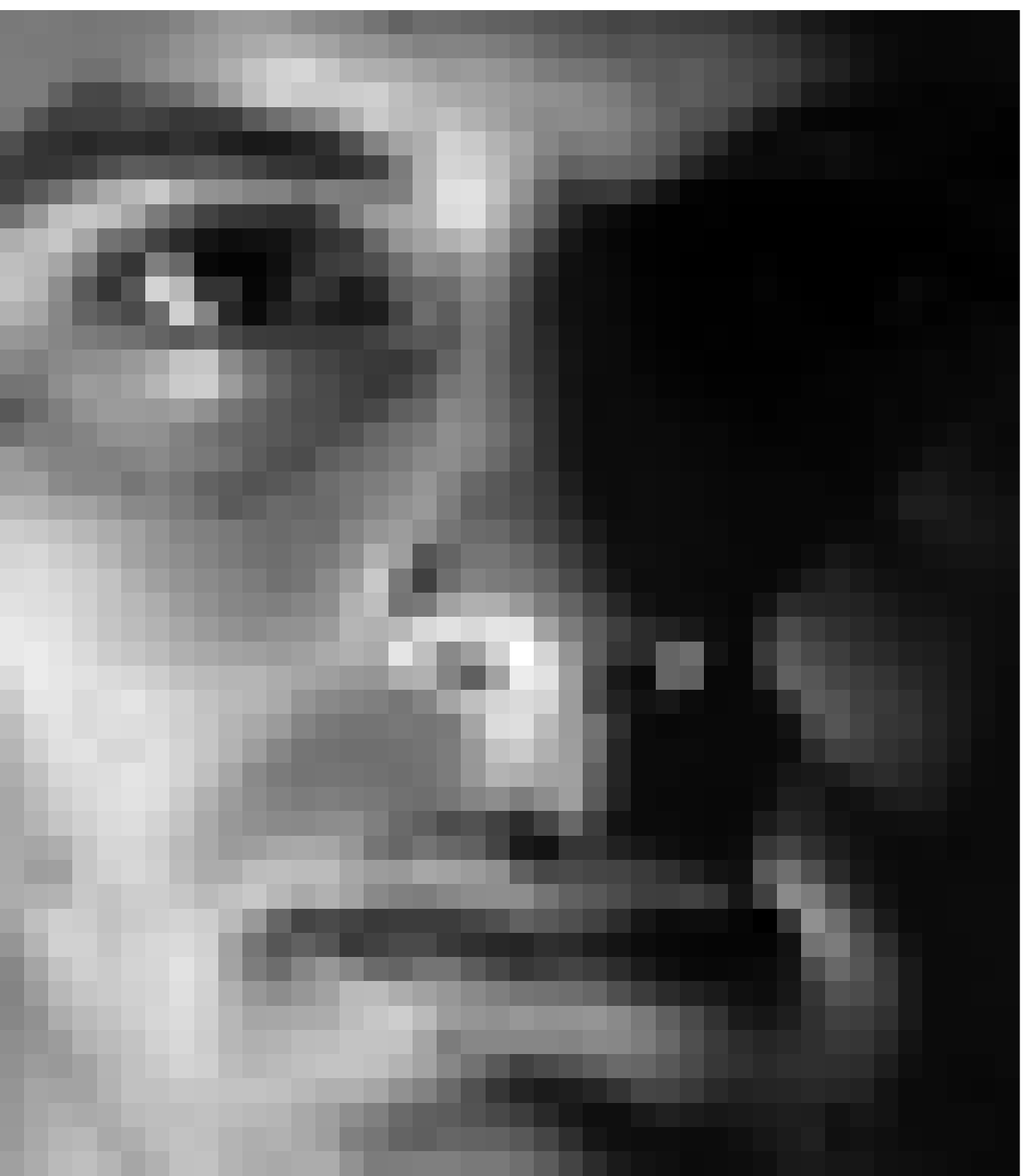}&
\includegraphics[width=.28\columnwidth]{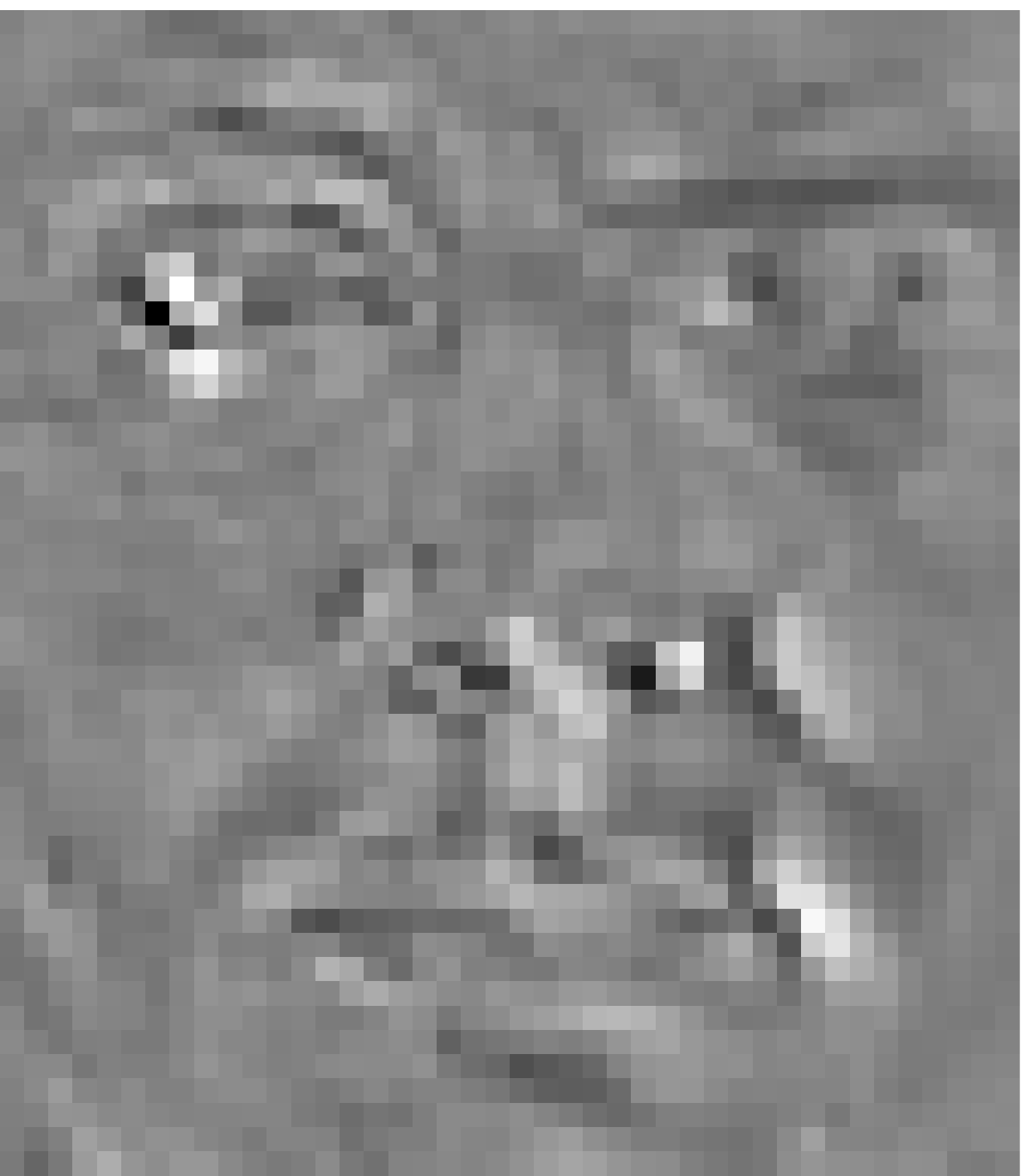}&
\includegraphics[width=.28\columnwidth]{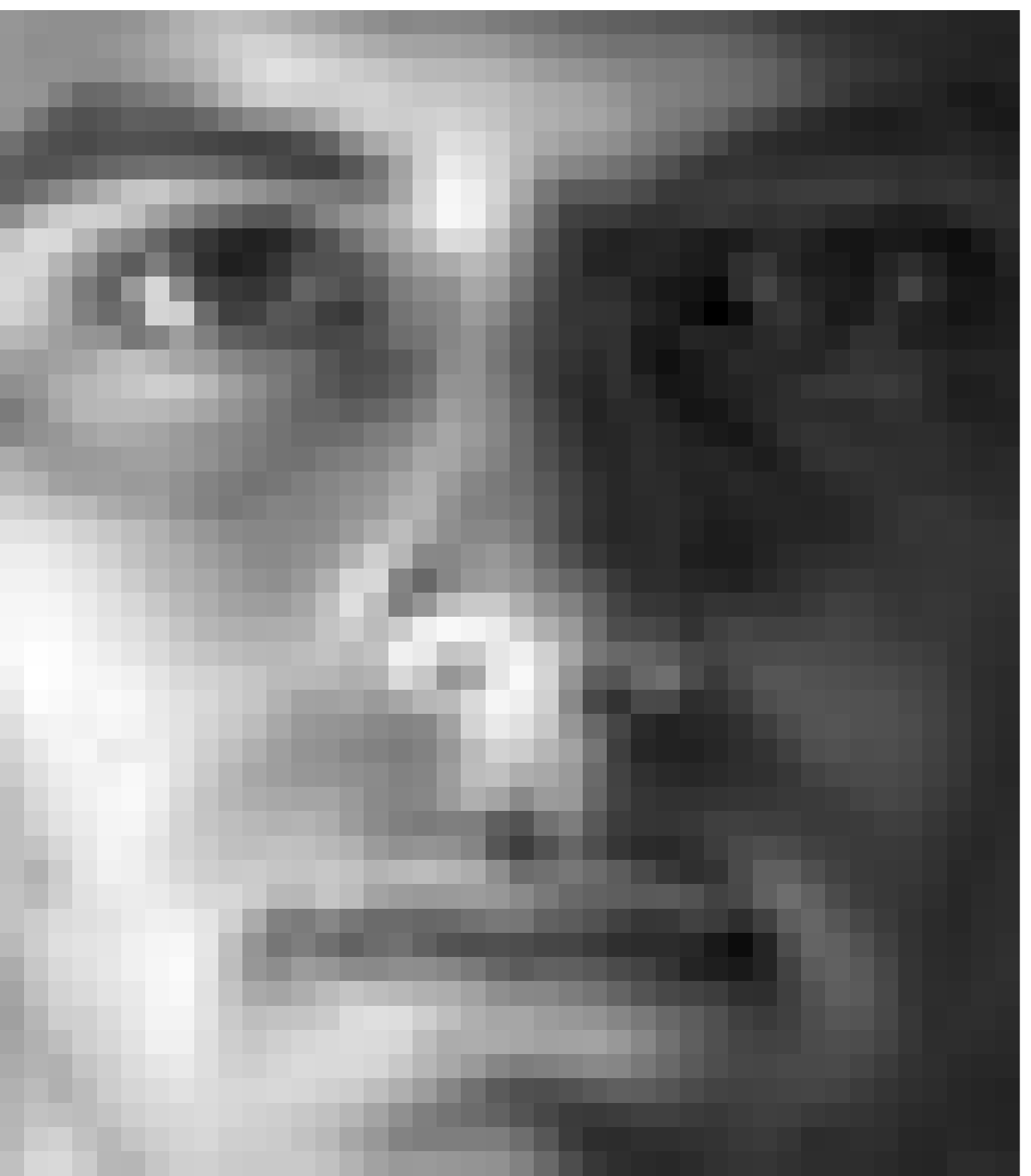}
\end{tabular}
\caption{ Recovery results of two face images. The three columns from left to right are the original image ($X$), the error matrix ($E$) and the recovered image ($XZ$), respectively.}
\label{faceerror}
\end{center}
\vskip -0.2in
\end{figure} 

\begin{table}[ht]
\caption{Clustering error rates (\%) on the EYaleB database.}
\vskip 0.15in
\begin{center}
\begin{small}
\begin{sc}
\resizebox{.45\textwidth}{!}{
\begin{tabular}{ c c c c c c c}
\hline\hline
Algorithm & LRR&SSC&LSA&LRSC& SCC&ARM\\
\hline
2 Subjects &&&&&&\\
 Mean&2.54&1.86&32.80&5.32&16.62&\textbf{1.51}\\
Median&0.78&\textbf{0.00}&47.66&4.69&7.82&0.78\\
\hline
3 Subjects &&&&&&\\
 Mean&4.21&3.10&52.29&8.47&38.16&\textbf{2.26}\\
Median&2.60&\textbf{1.04}&50.00&7.81&39.06&1.56\\
\hline
5 Subjects &&&&&&\\
Mean&6.90&4.31&58.02&12.24&58.90&\textbf{3.06}\\
Median&5.63&\textbf{2.50}&56.87&11.25&59.38&\textbf{2.50}\\
\hline
8 Subjects &&&&&&\\
 Mean&14.34&5.85&59.19&23.72&66.11&\textbf{3.70}\\
Median&10.06&4.49&58.59&28.03&64.65&\textbf{3.32}\\
\hline
10 Subjects &&&&&&\\
Mean&22.92&10.94&60.42&30.36&73.02&\textbf{3.85}\\
Median&23.59&5.63&57.50&28.75&75.78&\textbf{2.97}\\
\hline
\end{tabular}}
\label{table:face}
\end{sc}
\end{small}
\end{center}
\vskip -0.1in
\end{table}
 Table~\ref{table:face} provides the best performance of each method. As shown in the table, our proposed method has the lowest mean clustering error rates in all five settings. In particular, in the most challenging case of 10 subjects, the mean clustering error rate is as low as 3.85$\%$. The improvement is significant compared with other low rank representation based subspace clustering, i.e., LRR and LRSC. For example, 19$\%$ and 11$\%$ improvement over LRR can be observed in the cases of 10 and 8 subjects, respectively. This demonstrates the importance of accurate rank approximation. In addition, the error of LSA is large maybe because LSA is based on MSE. Since the MSE is quite sensitive to outliers, LSA will fail to deal with large outliers.   

\begin{figure}[ht]
\vskip 0.2in
\begin{center}
\centerline{\includegraphics[width=\columnwidth]{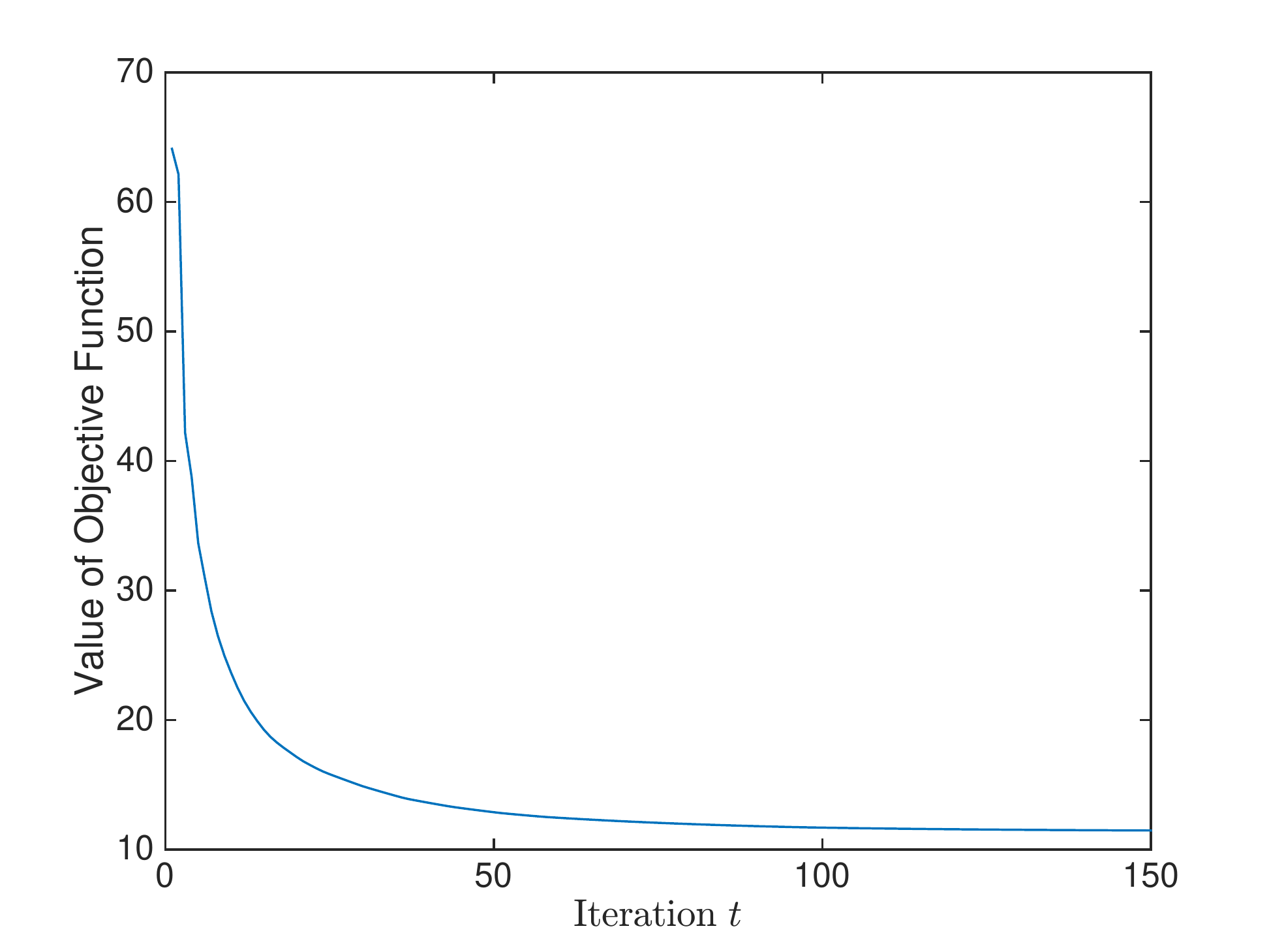}}
\caption{  Convergence curve of the objective function value in (\ref{original}).}
\label{func}
\end{center}
\vskip -0.2in
\end{figure} 

\begin{figure}[ht]
\vskip 0.2in
\begin{center}
\centerline{\includegraphics[width=\columnwidth]{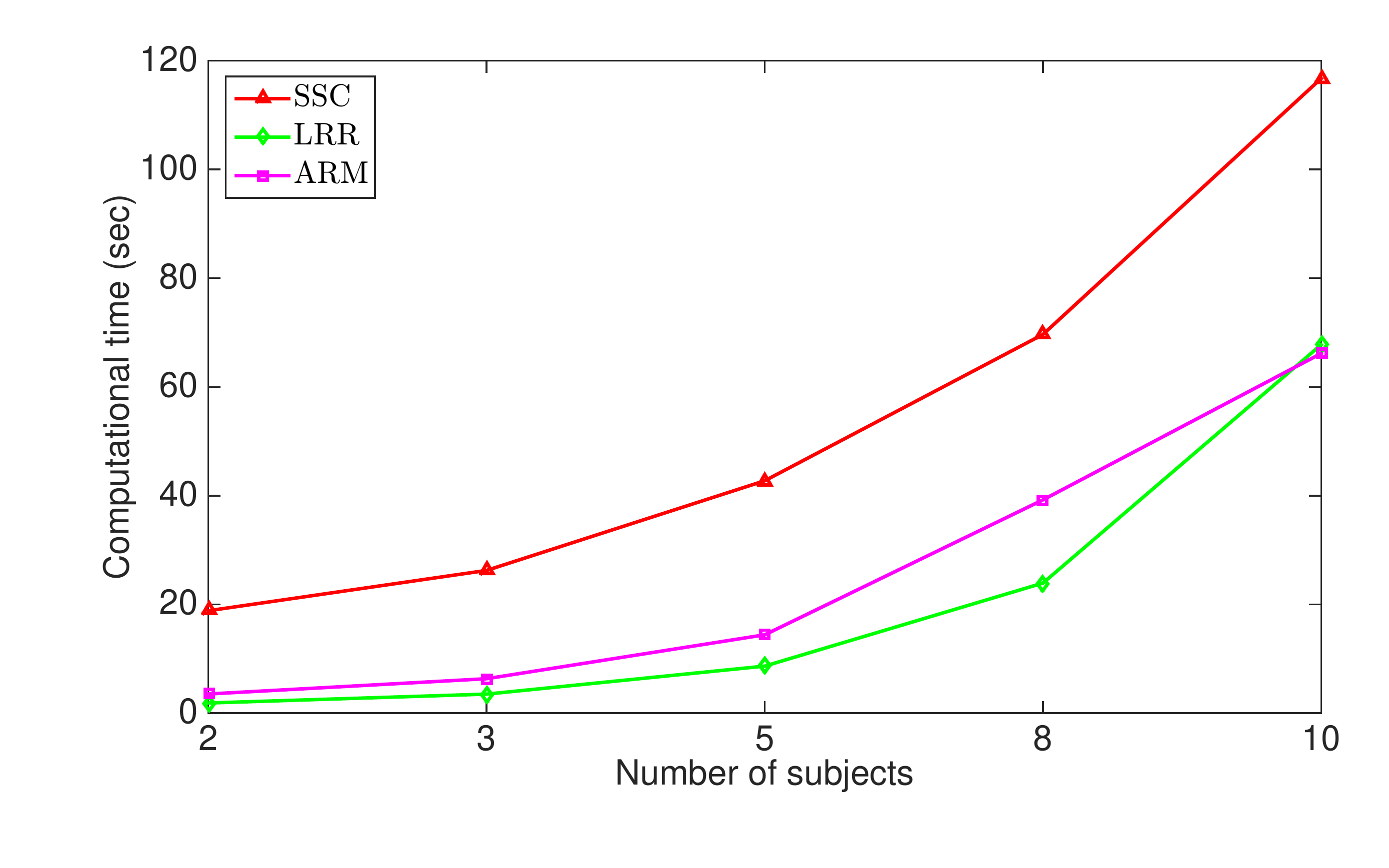}}
\caption{ Average computational time (sec) of the algorithms on the EYaleB database as a function of the number of subjects.}
\label{time}
\end{center}
\vskip -0.2in
\end{figure} 

Figure~\ref{10subject} shows the obtained affinity graph matrix $W$ for the five and ten subjects scenarios. We can see a distinct block-diagonal structure, which means that each cluster becomes highly compact and different subjects are well separated.
\begin{figure}[htb]
\vskip 0.2in
\begin{center}
\includegraphics[width=0.4\columnwidth]{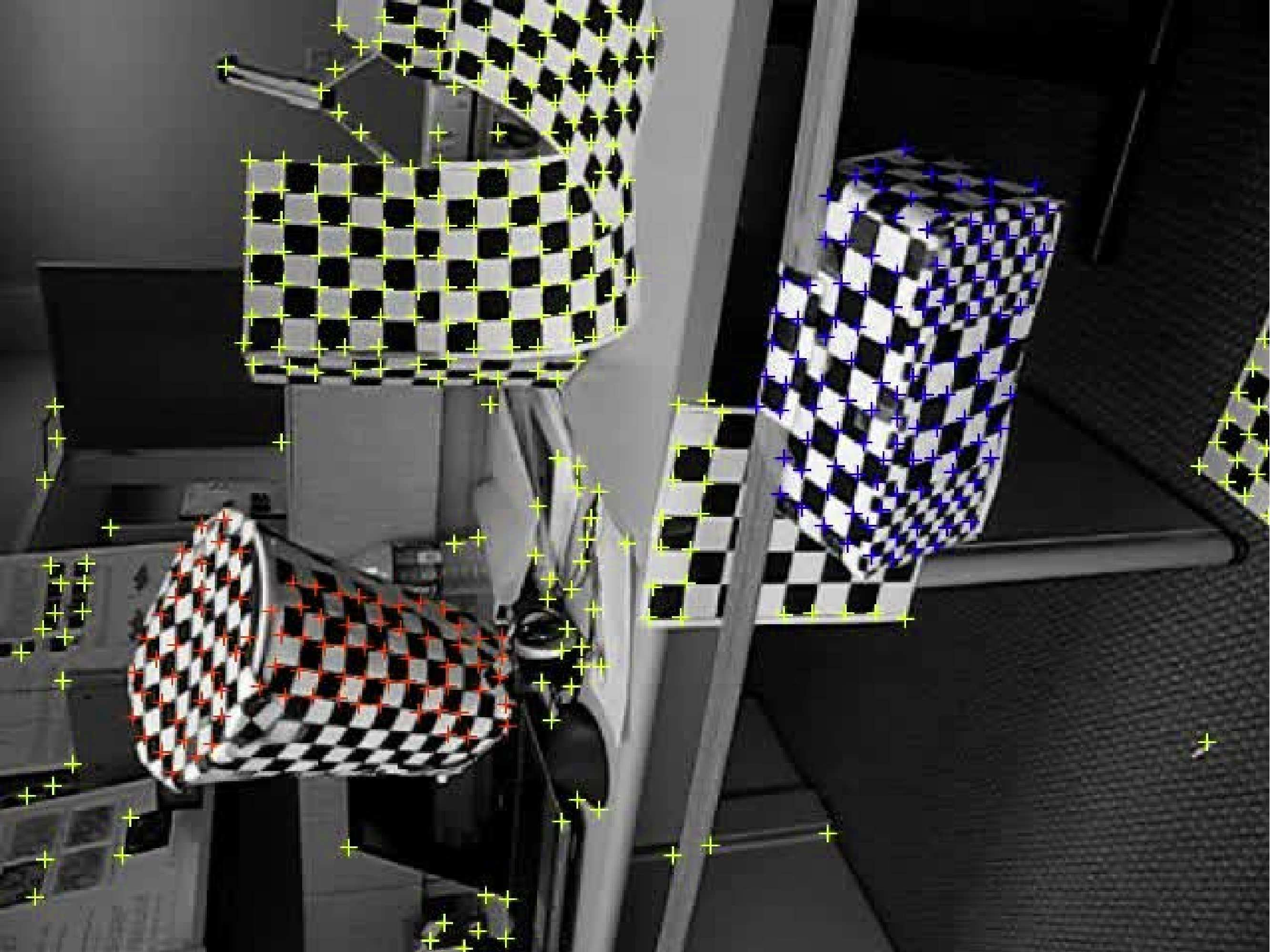}
\includegraphics[width=0.4\columnwidth]{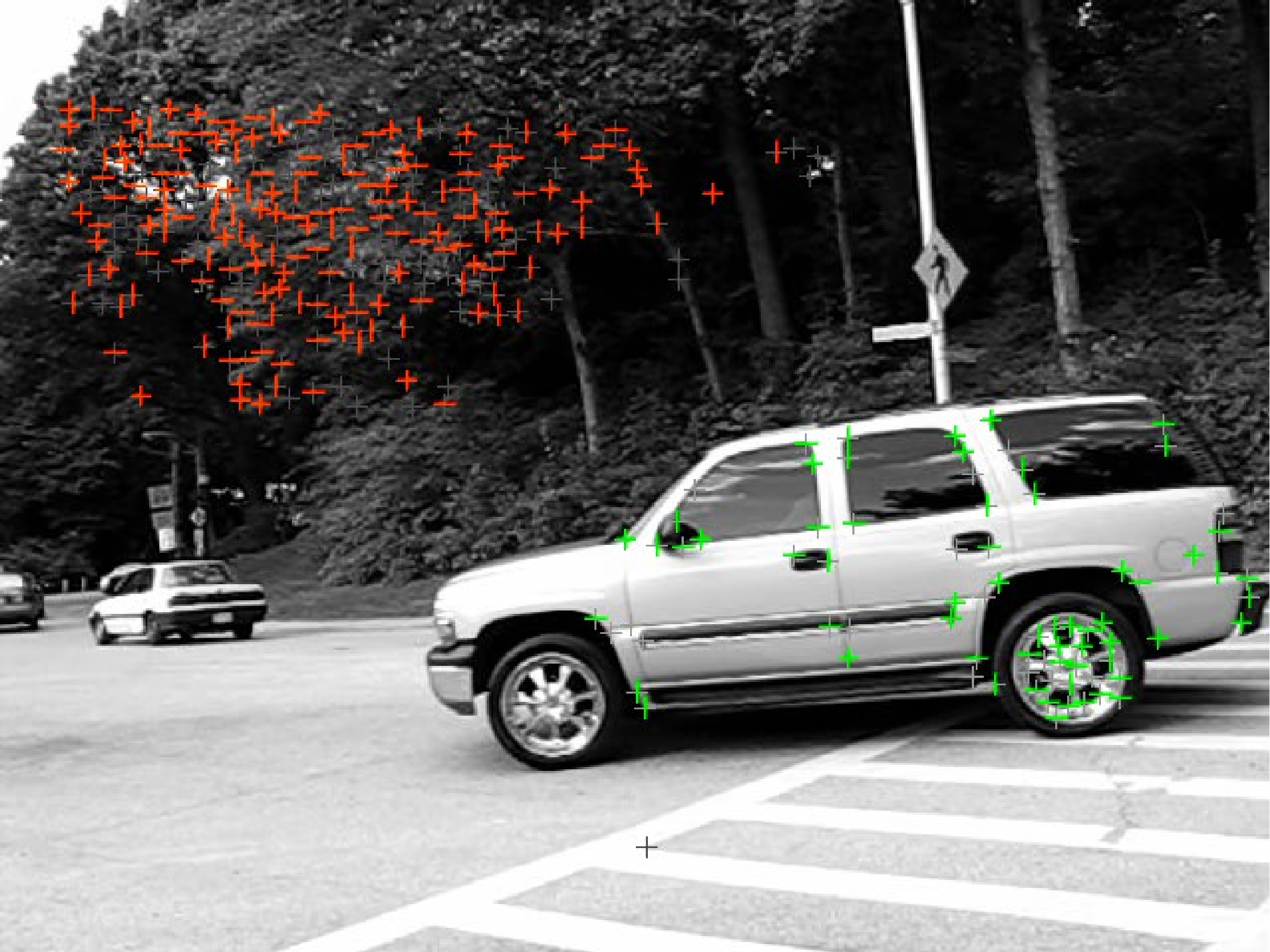}\\
\includegraphics[width=0.4\columnwidth]{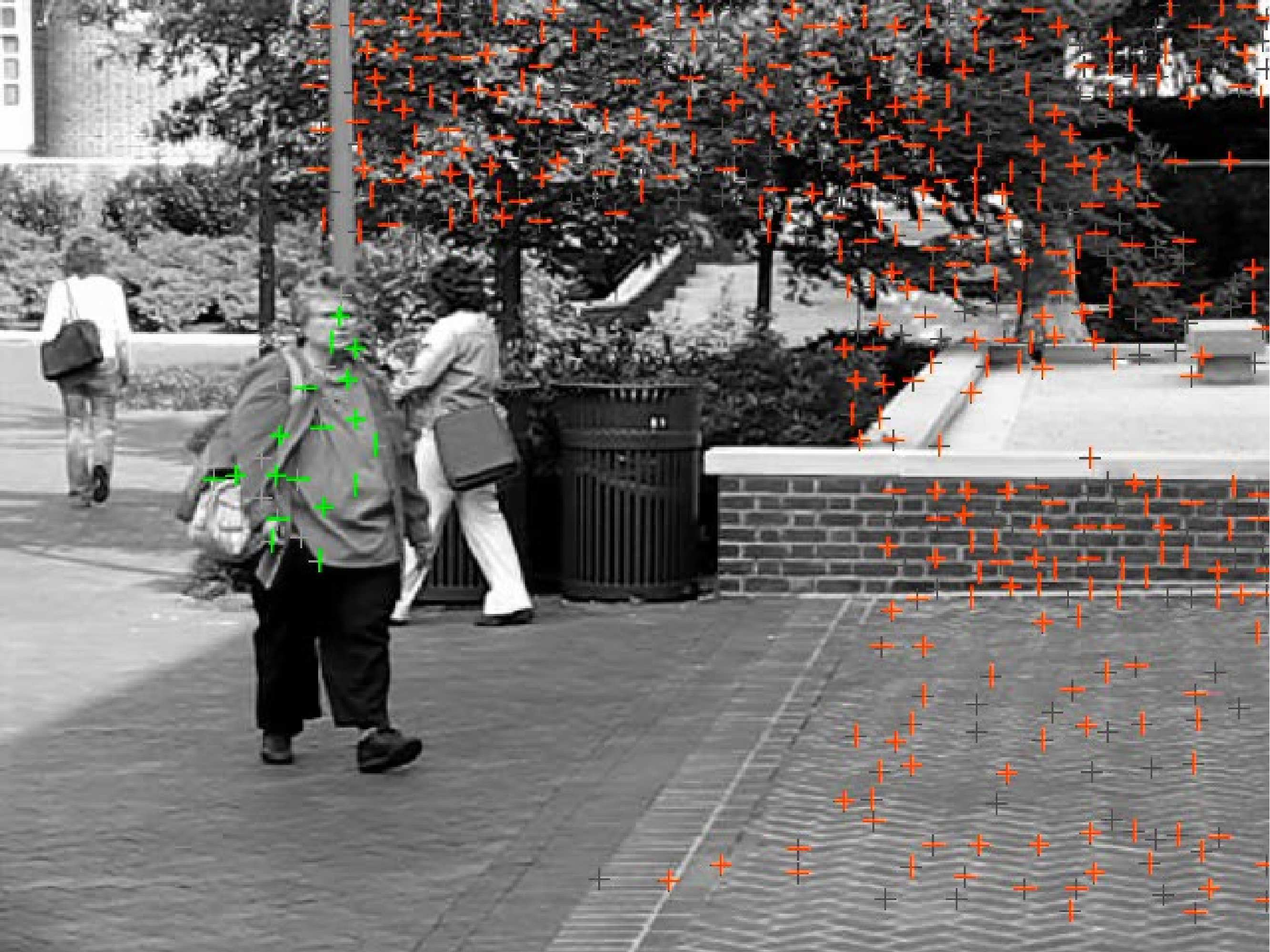}
\includegraphics[width=0.4\columnwidth]{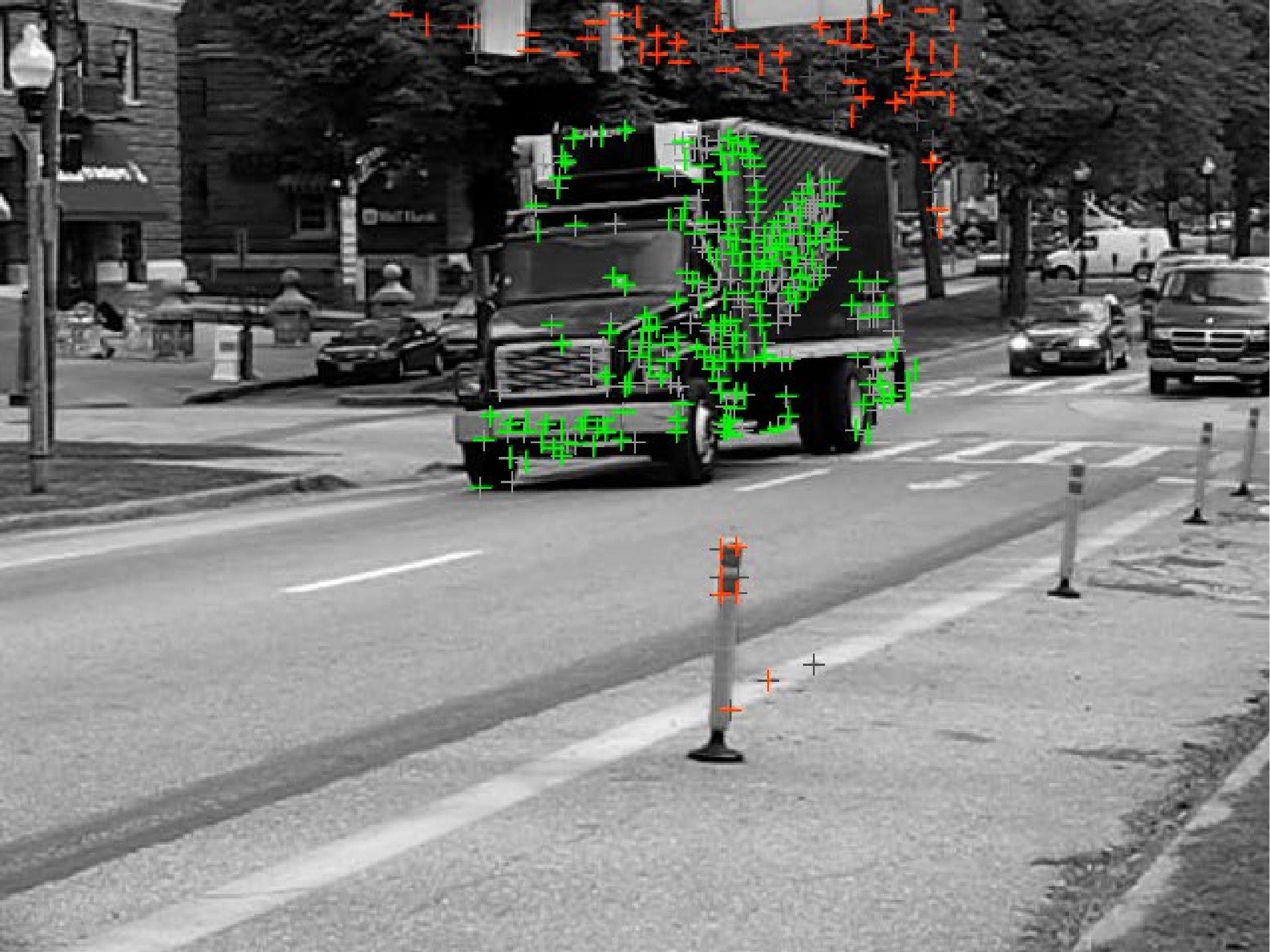}
\caption{Example frames from two video sequences of the Hopkins 155 database with traced feature points.}
\label{motionsample}
\end{center}
\vskip -0.2in
\end{figure}

In Figure~\ref{faceerror}, we present the recovery results of some sample faces from the 10-subject clustering case. We can see that the proposed algorithm has the benefit of removing the corruptions in data.

Figure~\ref{func} plots the progress of objective function values of (\ref{original}). It is observed that with more iterations, the value of objective function decreases monotonically. This empirically verifies the convergence of our optimization method.

We compare the average computational time of LRR, SSC, and ARM as a function of the number of subjects in Figure~\ref{time}. All the experiments are conducted and timed on the same machine with an Intel Xeon E3-1240 3.40GHz CPU that has 4 cores and 8GB memory, running Ubuntu and Matlab (R2014a). We can observe that the computational time of SSC is higher than LRR and ARM, while ARM is a little slower than LRR in most cases.
\subsection{Motion Segmentation}

\begin{table}[htb]
\caption{Segmentation error rates (\%) on the Hopkins 155 Dataset.}
\vskip 0.15in
\begin{center}
\begin{sc}
\resizebox{.45\textwidth}{!}{
\begin{tabular}{ c c c c c c c}
\hline\hline
Algorithm & LRR&SSC&LSA&LRSC&SCC&ARM  \\
\hline
2 Motions&&&&&&\\
 Mean&2.13&1.52&4.23&3.69&2.89&\textbf{1.48}  \\
Median&\textbf{0.00}&\textbf{0.00}&0.56&0.29&\textbf{0.00}&\textbf{0.00} \\
\hline
3 Motions &&&&&&\\
Mean&4.03&4.40&7.02&7.69&8.25&\textbf{1.49}  \\
Median&1.43&0.56&1.45&3.80&\textbf{0.24}&0.84  \\
\hline
All &&&&&&\\
Mean&2.56&2.18&4.86&4.59&4.10&\textbf{1.48} \\
Median&\textbf{0.00}&\textbf{0.00}&0.89&0.60&\textbf{0.00}&\textbf{0.00}  \\
\hline

\end{tabular}}
\label{table:motion}
\end{sc}
\end{center}
\vskip -0.1in
\end{table}

Motion segmentation involves segmenting a video sequence of multiple moving objects into multiple spatiotemporal regions corresponding to different motions. These motion sequences can be divided into three main categories: checkerboard, traffic, and articulated or non-rigid motion sequences. The Hopkins 155 dataset includes 155 video sequences of 2 or 3 motions, corresponding to 2 or 3 low-dimensional subspaces of the ambient space. Each sequence represents a data set and so there are 155 motion segmentation problems in total. Several example frames are shown in Figure~\ref{motionsample}. The trajectories are extracted automatically by a tracker, so they are slightly corrupted by noise. As in \cite{liu2013robust,liu2010robust}, $\left\|E\right\|_{2, 1}$ is adopted in the model. In this experiment, $\lambda=2$, $\mu^0=10$ and $\gamma=1.05$.

We use the original 2F-dimensional feature trajectories in our experiment. We show the clustering error rates of different algorithms in Table~\ref{table:motion}. ARM outperforms other algorithms in mean error rate. Especially, its all mean error rates are around $1.5\%$. This again demonstrates the effectiveness of using arctangent as a rank approximation.
\begin{figure}[ht]
\vskip 0.2in
\begin{center}
\centerline{\includegraphics[width=\columnwidth]{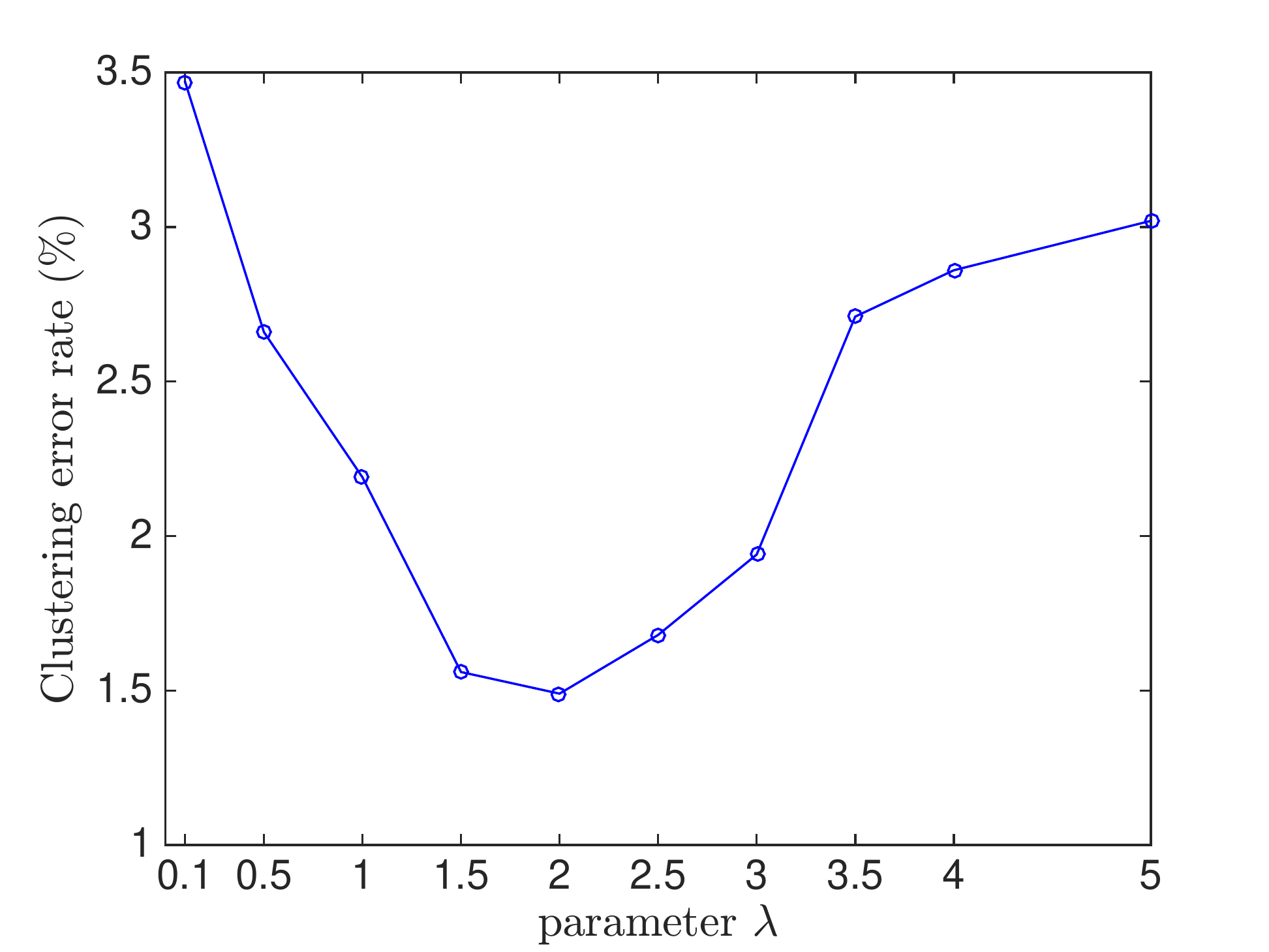}}
\caption{ The influence of parameter $\lambda$ of ARM on clustering error of Hopkins 155 database.}
\label{parametersense}
\end{center}
\vskip -0.2in
\end{figure}

Figure \ref{parametersense} shows the culstering error rate of ARM for different $\lambda$ over all 155 sequences. When $\lambda$ is between 1 and 3, the clustering error rate varies between $1.48\%$ and $2.19\%$. This demonstrates that ARM performs well under a pretty wide range of values of $\lambda$. This is another advantage of ARM over LRR \cite{liu2013robust}.
\section{Conclusion}
In this work, we propose to use arctangent as a concave rank approximation function. It has some nice properties compared with the standard nuclear norm. We apply this function to the low rank representation-based subspace clustering problem and develop an iterative algorithm for optimizing the associated objective function. Extensive experimental results demonstrate that, compared to many state-of-the-art algorithms, the proposed algorithm gives the lowest clustering error rates on many benchmark datasets. This fully demonstrates the significance of accurate rank approximation. Interesting future work includes other applications of the arctangent rank approximation; for example, matrix completion. Since LRR can only ensure its validity for independent subspace segmentation, it is worthwhile to investigate somewhat dependent yet possibly disjoint subspace clustering.
\vfill\eject
\section{Acknowledgments}
This work is supported by US National Science
Foundation Grants IIS 1218712. The corresponding author is Qiang Cheng.

\bibliographystyle{abbrv}
\bibliography{scref}  

\appendix
\section{Proof} 
\begin{theorem}
\label{firsthm}
For $\mu>0$ and $A\in \mathbf{\mathcal{R}}^{m\times n} $
, the following problem
\begin{equation}
Z^*=\argmin_Z F(Z)+\frac{\mu}{2}\left\|Z-A\right\|_F^2
\label{theoremprob}
\end{equation}
is solved by the vector minimization 
\begin{equation}
\label{vector}
\sigma^*=\argmin_{\sigma\geq0} f(\sigma)+\frac{\mu}{2}\|\sigma-\sigma_A\|_2^2,
\end{equation}
so that $Z^* = U diag(\sigma^*) V^T $ with the SVD of $A$ being \\$U diag(\sigma_A^*) V^T$. 
\end{theorem}
\begin{proof}
Let $A=U\Sigma_{A}V^T$ be the skinny SVD of $A$, then $\Sigma_{A}=U^TAV$. Denoting $X=U^TZV$ which has exactly the same singular values as $Z$, we have
\begin{flalign}
&F(Z)+\frac{\mu}{2}\|Z-A\|_F^{2}&\label{lower}\\
&= F(X)+\frac{\mu}{2}\|X-\Sigma_A\|_{F}^{2},&\label{unitary}\\
&= F(\Sigma_X)+\frac{\mu}{2}\|X-\Sigma_A\|_{F}^{2},&\label{inva}\\
&= F(\Sigma_X)+\frac{\mu}{2}\left(\|X\|_F^2+\|\Sigma_A\|_F^2-2\left\langle X,\Sigma_A\right\rangle \right),&\label{von}\\
&\geq  F(\Sigma_X)+\frac{\mu}{2}\left(\|\Sigma_X\|_F^2+\|\Sigma_A\|_F^2-2\left\langle \Sigma_X,\Sigma_A\right\rangle \right),&\label{eight}\\
&= F(\Sigma_X)+\frac{\mu}{2}\|\Sigma_X-\Sigma_A\|_F^{2},&\\
&= F(\Sigma_Z)+\frac{\mu}{2}\|\Sigma_Z-\Sigma_A\|_F^{2},&\label{ten}\\
&= f(\sigma)+\frac{\mu}{2}\|\sigma-\sigma_{A}\|_2^2,&\label{eleven}\\
&\ge f(\sigma^*) + \frac{\mu}{2} \|\sigma^* - \sigma_{ A}\|_2^2. 
\end{flalign}
In the above, (\ref{unitary}) holds because the Frobenius norm is unitarily invariant; (\ref{inva}) holds because $F(X)$ is unitarily invariant; (\ref{eight}) is true by von Neumann's trace inequality; and (\ref{ten}) holds because of the definition of $X$.  Therefore, (\ref{ten}) is a lower bound of (\ref{lower}). Note that the equality in (\ref{eight}) is attained if $X = \Sigma_X$. Because $\Sigma_Z = \Sigma_X = X = U^T Z V$, the SVD of $Z$ is $Z = U \Sigma_Z V^T$. By minimizing (\ref{eleven}), we get $\sigma^*$. Therefore, eventually we get 
$Z^* = U diag(\sigma^*) V^T$, which  is the minimizer of problem (\ref{theoremprob}). 
\end{proof}
\vfill\eject
\section{Theorem and Lemmas}
\begin{theorem} \cite{lewis2005nonsmooth}
Suppose $F: \mathbf{\mathcal{R}}^{m\times n}\rightarrow \mathbf{\mathcal{R}}$ is represented as $F(X)=f \circ \sigma(X)$, and $f :\mathbf{\mathcal{R}}^{ n}\rightarrow \mathbf{\mathcal{R}}$ is absolutely symmetric and differentiable, where $X\in\mathbf{\mathcal{R}}^{m\times n} $ with SVD 
 $X=U diag(\sigma) V^T$, the gradient of $F(X)$ at $X$ is
\begin{equation}
\frac{\partial F(X)}{\partial X}=U diag(\theta) V^T,
\label{takederi}
\end{equation}
where $\theta=\frac{\partial f(y)}{\partial y}|_{y=\sigma (X)}$.
\end{theorem}

\begin{lemma}
\cite{beck2009fast} For $\mu>0$, and $K\in\mathbf{\mathcal{R}}^{s\times t}$, the solution of the problem  
\begin{equation*}
\min_L\vspace{.2cm} \mu\|L\|_1+\frac{1}{2}\|L-K\|_F^2, 
\end{equation*}
is given by $L_\mu(K)$, which is defined component-wisely by
\begin{equation*}
[L_{\mu}(K)]_{ij}=max\{|K_{ij}|-\mu,0\}\cdot sign(K_{ij}).
\end{equation*}
\end{lemma}
\begin{lemma} \label{lemma:solve_l2l1} \cite{yang2009fast} Let $H$ be a given matrix. If the optimal solution to
\begin{eqnarray*}
\min_{W} \alpha\left\|W\right\|_{2,1} + \frac{1}{2}\left\|W-H\right\|_F^2
\end{eqnarray*}
is $W^*$, then the $i$-th column of $W^*$ is
\begin{eqnarray*}
[W^*]_{:,i}=\left\{
\begin{array}{ll} \frac{\left\|H_{:,i}\right\|_2-\alpha}{\left\|H_{:,i}\right\|_2}H_{:,i}, & \mbox{if $\left\|H_{:,i}\right\|_2>\alpha$};\\
0, & \mbox{otherwise.}
\end{array}\right.
\end{eqnarray*}
\end{lemma}
\end{document}